\documentclass[10pt,twocolumn,letterpaper]{article}

\usepackage{cvpr}              

\usepackage{graphicx}
\usepackage{amsmath}
\usepackage{amssymb}
\usepackage{booktabs}
\usepackage{algorithm}
\usepackage{algorithmic}
\usepackage{microtype}
\usepackage{amsmath, amssymb, amsthm}
\usepackage{mathtools}
\usepackage{mathrsfs}
\usepackage{booktabs}
\usepackage{xcolor}
\usepackage{array}

\usepackage{amsmath}

\usepackage{makecell}
\usepackage{arydshln}
\usepackage{cellspace}
\usepackage{subcaption}

\usepackage{times}
\usepackage{latexsym}
\usepackage{graphicx}
\usepackage{amsmath}
\usepackage{multirow}
\usepackage{booktabs} 
\usepackage{tablefootnote}
\usepackage{colortbl} 
\usepackage{paralist} 
\usepackage{soul} 
\usepackage{xspace}
\usepackage{tabularx}
\newcommand{\ourmodel}{\textsc{DynaFed}\xspace}

\def\bbE{\mathbb{E}}

\def\cL{\mathcal{R}}
\def\cJ{\mathcal{J}}
\def\bg{\boldsymbol{g}}
\def\bx{\boldsymbol{x}}
\def\by{\boldsymbol{y}}

\def\bv{\boldsymbol{v}}
\def\bw{\boldsymbol{w}}

\def\bw{\boldsymbol{w}}

\def\bX{\boldsymbol{X}}

\def\cA{\mathcal{A}}

\def\cD{\mathcal{D}}

\def\cL{\mathcal{L}}
\def\cM{\mathcal{M}}

\def\cI{\mathcal{I}}
\def\cX{\mathcal{X}}

\newtheorem{lemma}{Lemma}

\newtheorem{remark}{Remark}

\newcommand{\w}{\boldsymbol{w}}

\definecolor{lightgrey}{HTML}{dcdbdb}
\sethlcolor{lightgrey}

\usepackage[pagebackref,breaklinks,colorlinks]{hyperref}

\usepackage[capitalize]{cleveref}
\crefname{section}{Sec.}{Secs.}
\Crefname{section}{Section}{Sections}
\Crefname{table}{Table}{Tables}
\crefname{table}{Tab.}{Tabs.}

\newtheorem{assumption}{Assumption}
\newtheorem{theorem}{Theorem}
\newtheorem{proposition}{Proposition}

\def\cvprPaperID{5089} 
\def\confName{CVPR}
\def\confYear{2023}

\allowdisplaybreaks
\begin{document}

\title{\ourmodel: Tackling Client Data Heterogeneity with Global Dynamics}
\author{
Renjie Pi$^1$\footnotemark[1]\
\quad Weizhong Zhang$^1$\footnotemark[1]\
\quad Yueqi Xie$^1$\footnotemark[1]\
\quad Jiahui Gao$^2$
\quad Xiaoyu Wang$^1$ \\
\quad Sunghun Kim$^1$
\quad Qifeng Chen$^1$\\
$^1$The Hong Kong University of Science and Technology \\
$^2$The University of Hong Kong\\
}

\maketitle
\renewcommand{\thefootnote}{\fnsymbol{footnote}}
\footnotetext[1]{Joint first authors}
\begin{abstract}
The Federated Learning (FL) paradigm is known to face challenges under heterogeneous client data. Local training on \textit{non-iid} distributed data results in deflected local optimum, which causes the client models drift further away from each other and degrades the aggregated global model's performance. A natural solution is to gather all client data onto the server, such that the server has a global view of the entire data distribution. Unfortunately, this reduces to regular training, which compromises clients' privacy and conflicts with the purpose of FL. In this paper, we put forth an idea to collect and leverage global knowledge on the server without hindering data privacy. We unearth such knowledge from the dynamics of the global model's trajectory. Specifically, we first reserve a short trajectory of global model snapshots on the server. Then, we synthesize a small pseudo dataset such that the model trained on it mimics the dynamics of the reserved global model trajectory. Afterward, the synthesized data is used to help aggregate the deflected clients into the global model. We name our method \ourmodel, which enjoys the following advantages: 1) we do not rely on any external on-server dataset, which requires no additional cost for data collection; 2) the pseudo data can be synthesized in early communication rounds, which enables \ourmodel to take effect early for boosting the convergence and stabilizing training; 3) the pseudo data only needs to be synthesized once and can be directly utilized on the server to help aggregation in subsequent rounds. Experiments across extensive benchmarks are conducted to showcase the effectiveness of \ourmodel. We also provide insights and understanding of the underlying mechanism of our method.

\end{abstract}
\begin{figure}[t]
        \centering
        \includegraphics[width=1.0\linewidth]{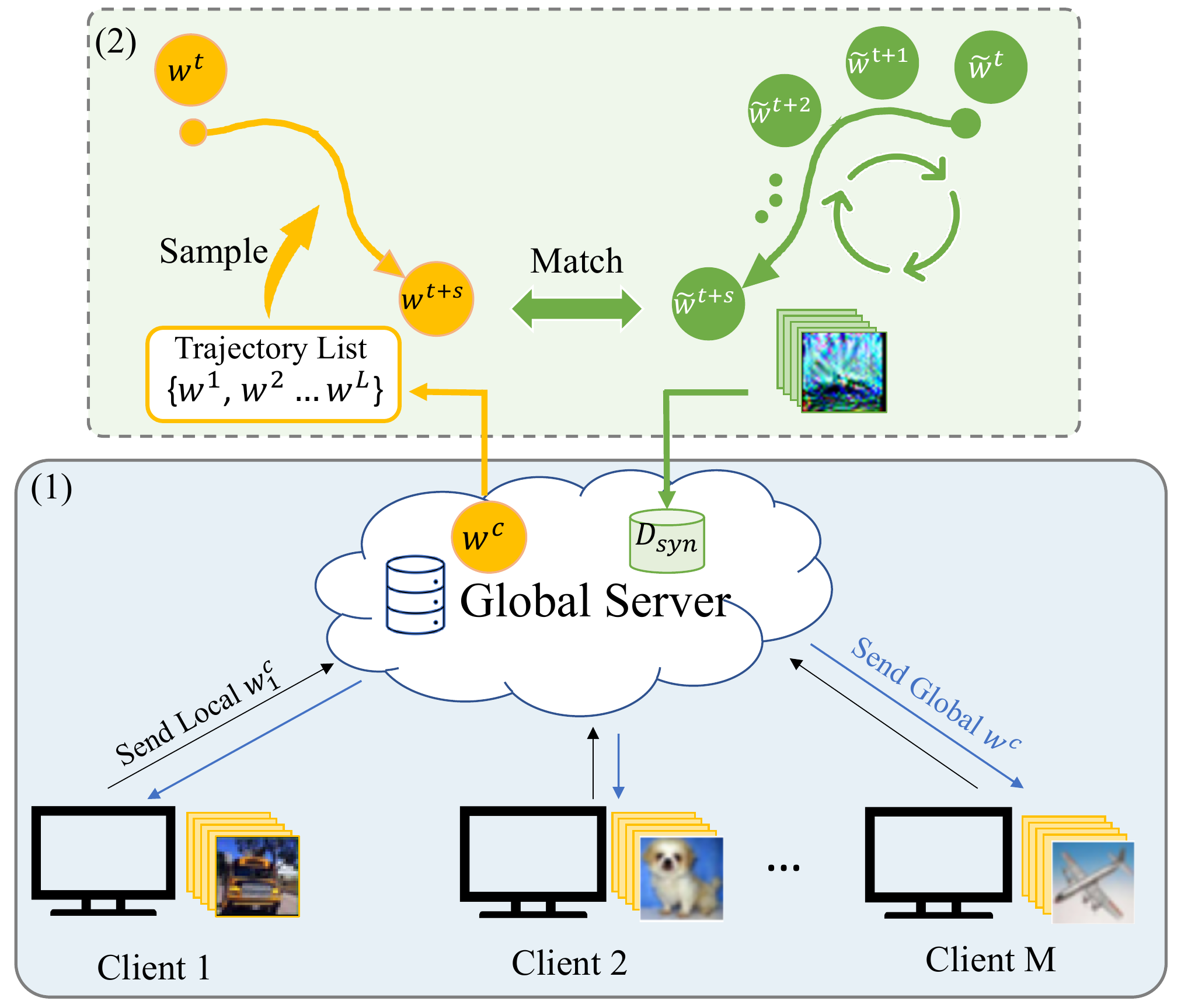}
    \caption{Illustration of \ourmodel. Firstly, we run the standard FedAvg for $L$ communication rounds and save the checkpoints to form a trajectory of the global model at the server. Then, we synthesize a pseudo dataset $\cD_\text{syn}$, with which the network can be trained to mimic the dynamics of the global trajectory. In this way, the knowledge that captures the essential information about the global data distribution is transferred from the global dynamics to $\cD_\text{syn}$. Afterward, $\cD_\text{syn}$ is adopted to help aggregate the deflected clients into the global model at the server.}
     \vspace{-0.33cm}
    \label{fig:framework}
\end{figure}
\section{Introduction}
Federated learning (FL) has become a popular distributed training paradigm to alleviate the server's computational burden and preserve clients' data privacy~\cite{bonawitz2019towards,li2020federated,kairouz2021advances, xie2022optimizing}. In the FL paradigm, the clients only have access to their private datasets, while the server is responsible for aggregating the clients' updates into a global model. The most prevalent approaches in FL are based on local-SGD \cite{mcmahan2017communication} (also referred to as FedAvg), where the client model is updated locally for multiple steps before being sent and merged on the server. Such approaches save communication costs and perform well given the client data are \textit{iid}-distributed. However, in real-world applications such as healthcare \cite{rieke2020future,sheller2020federated, jochems2017developing,kairouz2021advances} and bio-metrics \cite{aggarwal2021fedface}, the client data usually demonstrates heterogeneity (highly \textit{non-iid}), which deflects the local optimum from the global optimum \cite{zhao2018federated, li2019convergence, karimireddy2019scaffold} and makes the locally trained clients biased. Therefore, naively averaging the client models results in slow convergence and 
performance drop.

To alleviate the difficulty of training under heterogeneity, a few lines of work have been frequently discussed. 
The first line attempts to modify the local training process, including imposing regularization on the client models~\cite{li2020federated,karimireddy2019scaffold,li2021model,acar2021federated} and data sharing or augmentation~\cite{shin2020xor, oh2020mix2fld,YoonSHY21,zhao2018federated}.
However, these solutions have a high requirement for the server's control of local clients. 
An orthogonal line focuses on refining the global model in the server aggregation process~\cite{wang2020tackling,yeganeh2020inverse, xiao2021novel,lin2020ensembleFedDF, chen2020fedbe}.
The majority of these methods typically require a large external dataset on the server, then use it to align the outputs of the global model with that of the client ensemble~\cite{lin2020ensembleFedDF, chen2020fedbe, xiao2021novel,Fedaux}. Unfortunately, such a large-scale task-related dataset is often hard to acquire in reality.
To circumvent this limitation, a few data-free knowledge distillation (KD) approaches are recently proposed~\cite{zhu2021data, zhang2022fine}.  These methods attempt to transfer the knowledge contained in the global model to a generator, which is subsequently leveraged to produce pseudo data to either help local training at the clients \cite{zhu2021data} or finetune the global model at the server \cite{zhang2022fine}. It is clear that these methods require a global model with reasonable performance to ensure the generation of helpful pseudo data.
However, such requirement is hard to achieve in practice since the global model often performs poorly under heterogeneity, especially in the early rounds of training.
Other solutions include personalized FL~\cite{t2020personalized,hanzely2020lower,li2021ditto,kulkarni2020survey}, simlarity clustering~\cite{dennis2021heterogeneity,ghosh2019robust,briggs2020federated}, meta learning~\cite{li2022federated, jiang2019improving}, etc.


Even though the above-mentioned works propose techniques to alleviate the challenges posed by heterogeneity to some extent, they do not tackle the issue from its root cause: the data is unevenly scattered at different clients and is kept locally due to privacy concerns, which is also the main obstacle for training an 
accurate global model.
Ideally, imagine if we can collect all the client data to the server, then training can be directly conducted at the server, and the heterogeneity issue no longer exists. However, this reduces to regular training and conflicts with the original purpose of FL to protect client privacy. We then raise a natural question: \textit{is it possible to derive the essential information about the global data distribution on the server to help training without compromising client privacy?}

Despite the global model typically performing poorly due to heterogeneity, the changes in its parameters are steered jointly by the data scattered at different clients. Therefore, the update dynamics of the global model contain knowledge about global data distribution. Driven by this intuition, we propose \ourmodel to explicitly unearth such knowledge hidden in the global dynamics and transfer it to a pseudo dataset $\cD_\text{syn}$ at the server. $\cD_\text{syn}$ can then approximate the global data distribution on the server to aid aggregation. More specifically, inspired by recent works in dataset condensation \cite{zhao2021dataset, cazenavette2022dataset, wang2022cafe}, we formulate the data synthesis process into a learning problem, which minimizes the distance between the trajectory trained with $\cD_\text{syn}$ and the global model trajectory derived with $\cD$. Fine-tuning the aggregated global model with  $\cD_\text{syn}$ effectively alleviates the performance degradation caused by deflected clients. An appealing feature of our \ourmodel is that the data can be synthesized using just the global model's trajectory of the first few rounds, which enables $\cD_\text{syn}$ to take effect and help aggregation from early rounds. In addition, the synthesizing process only needs to be conducted once in practice, after which the derived $\cD_\text{syn}$ can be directly applied in subsequent rounds to help aggregate the deviated client models.

Notably, our framework can be readily applied to the majority of FL approaches, since we rely on only the history of the global model's parameters for synthesizing $\cD_\text{syn}$, which is available in the conventional setting of FedAvg-based methods. Furthermore, because we extract global knowledge using the global dynamics, rather than any client-specific information as in~\cite{yin2018byzantine, huang2021evaluating,hatamizadeh2022gradvit, jeon2021gradient,fowl2021robbing}, the derived $\cD_\text{syn}$ comprises of information mixed with the entire global data distribution, thus prevents leakage of client privacy.

Our \ourmodel possesses the following advantages compared with previous approaches: 1) It leverages the knowledge of the global data distribution to alleviate the aggregation bias of the global model without depending on any external datasets; 2) \ourmodel is able to generate informative data in the early rounds of federated learning, which significantly helps convergence and stabilizes training in subsequent rounds; 3) Compared with \cite{zhu2021data, zhang2022fine}, which need to keep updating the generator throughout all communication rounds, the data synthesis process in our method only needs to be done once, which reduces the computational overhead.
In summary, we make the following contributions: 
\begin{itemize}
\setlength{\itemsep}{0pt}
\setlength{\parsep}{0pt}
\setlength{\parskip}{0pt}
    \item We propose a practical approach named \ourmodel for tackling the heterogeneity problem, which extracts and exploits the hidden information from the global model's trajectory. In this way, the server can access the essential knowledge of the global data distribution to reinforce aggregation;
    \item We experimentally show the synthesized dataset helps stabilize training, boost convergence and achieve 
    significant performance improvement
    under heterogeneity;
    \item We provide insights and detailed analysis into the working mechanisms of the proposed \ourmodel both experimentally and theoretically.
\end{itemize}

\section{Related Work}

\paragraph{Regularization-based Methods} FedAvg\cite{mcmahan2017communication} is the most widely used technique in FL, which periodically aggregates the local models to the global model in each communication round. FedProx \cite{li2020federated} proposes to impose a proximal term during local training, such that the local model does not drift too far from its global initialization; Scaffold \cite{karimireddy2019scaffold} introduces a control variate and variance reduction to alleviate the drift of local training. Moon \cite{li2021model} proposes to leverage the similarity between model representations to regularize the client local training. FedDyn \cite{acar2021federated} introduces the linear and quadratic penalty terms to correct the clients' objective during local training. These methods are orthogonal to our approach and can be jointly used.
\vspace{-0.35cm}
\paragraph{Data-Dependent Knowledge Distillation Methods} This line of work attempts to distill client ensemble knowledge into the global model. \cite{lin2020ensembleFedDF} proposes to use an unlabeled external dataset on the server to match the global model's outputs with that of the client ensemble. On top of this, \cite{chen2020fedbe} further proposes to sample and combine higher-quality client models via a Bayesian model ensemble. Subsequently, some advanced techniques, such as pre-training~\cite{Fedaux} and weighted consensus distillation scheme~\cite{cho2022heterogeneous} are proposed. These methods typically require a large amount of data following similar distribution as the task data, which is usually hard to acquire in practice. Besides, these methods need to conduct KD with a large dataset in every communication round, which introduces prohibitive computational overhead.
\paragraph{Data-free Knowledge Distillation Methods} Recently, a few works have proposed to perform KD in a data-free manner by synthesizing data with generative models~\cite{zhu2021data, zhang2022fine}. \cite{zhu2021data} proposes to train a lightweight generator on the server, which produces a feature embedding conditioned on the class index. The generator is then sent to the clients to regularize local training. \cite{zhang2022fine} trains a class conditional GAN\cite{goodfellow2020generative} on the server, where the global model acts as the discriminator. The pseudo data is then used to finetune the global model. These methods all depend on the global model for training the generator. Unfortunately, the model performance is often poor under high data heterogeneity, which makes the quality of the pseudo data questionable. On the other hand, due to the use of update dynamics of the global model rather than an individual model, our method can synthesize high-quality data containing rich global information even if the global model performs poorly.

\section{Preliminary}
\textbf{Federated Learning.} 
Suppose we have $M$ clients in a federated learning system. 
For each client $m\in[M]$, a private dataset ${\mathcal{D}}_m = \{(\bx_i, y_i)\}_{i=1}^{|\mathcal{D}_m|}$ is kept locally. 
The overall optimization goal of the FL system is to jointly train a global model which performs well on the combination of local data, denoted as $\mathcal{D} = \cup_{m=1}^M \mathcal{D}_m$, where $\cD$ is from a global distribution.
Let $\alpha_m = \frac{|{\mathcal{D}}_m|}{|\mathcal{D}|}$ denote the portion of data samples on client $m$. Let $\bw \in \mathbb{R}^d$ denote the model parameter to optimize, and $ \mathcal{L}_m{(\bw,\mathcal{D}_m)} = \frac{1}{|\mathcal{D}_m|}\sum_{\xi \in \mathcal{D}_m} \ell(\bw, \xi)$ denote the empirical risk with the loss function $\ell(\cdot, \cdot)$.
The optimization problem of a generic FL system can be formulated  as follows:
\begin{equation}
\small
    \min_{\bw}{\mathcal{L}(\bw,\mathcal{D}) ={\sum_{m=1}^{M} \alpha_{m}{\mathcal{L}_m{(\bw,\mathcal{D}_m)}}}}.\label{eqn:obj}
\end{equation}

\textbf{FedAvg.} 
The main-stream solutions of Eqn. \ref{eqn:obj} rely on local-SGD to reduce the communication cost of transferring gradients. FedAvg~\cite{mcmahan2017communication} is the most prevalent approach, which uses a weighted average to aggregate the locally trained models into a global model in each communication round.
Typically, not all the clients participate in every communication round. Suppose $\mathcal{M}^c \subset [M]$ is the set of the participated clients in the $c$-th round, and $P^c=\sum_{m\in \mathcal{M}^c} \alpha_m$. The aggregation process for the global model $\bw^{c+1}$ at the end of $c$-th communication round can be formulated as: 
\begin{equation}
\small
    \bw^{c+1}=\frac{1}{P^c}\sum_{m\in \mathcal{M}^c} \alpha_m{\bw_m^{c}},\label{eq_FedAVG}
\end{equation}
where $\bw_m^{c}$ denotes the client $m$'s locally trained model. After the aggregation, the updated global model $\bw^{c+1}$ is then distributed to and client and utilized to initiate the $c+1$-th round of the training.

\section{Proposed Method}\label{sec:proposed_method}
In this section, we introduce our proposed method \ourmodel. The overall framework is illustrated in Figure \ref{fig:framework}. Firstly, we collect a trajectory of  the global model's updates in the early phase of federated training, with which we construct a synthetic dataset $\cD_\text{syn}=\{\bX, \by\}$ on the server side. Then, we utilize $\cD_\text{syn}$ to aid the server-side aggregation, which effectively helps recover the performance drop caused by deflected local models.

\subsection{Acquiring Global Knowledge by Data Synthesis}
Our goal is  to construct a pseudo dataset $\cD_\text{syn}$, which achieves a similar effect during training as the global dataset $\cD$. In other words, the trajectory of a network trained with $\cD_\text{syn}$ should have similar dynamics with the trajectory trained with $\cD$. 
To be precise, we denote the trajectory trained on $\cD$ as a sequence $\{\bw^{\text{t}}\}^L_{t=0}$, 
where L is the length of the trajectory.  In order to reduce the number of unrolling steps during optimization, we align the trajectory in a segment-by-segment manner. Without loss of generality, we consider a segment from $\bw^t$ to $\bw^{t+s}$, then the problem becomes the following: starting from $\bw^t$, the network should arrive at a place close to $\bw^{t+s}$ after being trained for $s$ steps on $\cD_\text{syn}$. We further formulate the data synthesis task into a learning problem as follows:
\small{
\begin{gather}
\min_{\bX, \by}\bbE_{t\sim U(1, L-s)}
[d(\Tilde{\bw}, \bw^\text{t+s})]\label{eqn:outer_loop}
\\
s.t. ~\Tilde{\bw} = \cA(\bX, \by, \bw^\text{t}, s)\label{eqn:inner_loop}
\end{gather}}In the inner loop expressed by Eqn.\ref{eqn:inner_loop}, we run the trainer $\cA(\cdot)$ that trains a neural network initialized from $\bw^\text{t}$ on the synthetic dataset $\cD_\text{syn}=\{\bX, \by\}$ for $s$ steps, which arrives at $\Tilde{\bw}$. In the outer loop, we minimize the distance between $\bw^\text{t+s}$ and $\Tilde{\bw}$, denoted as $d(\Tilde{\bw}, \bw^\text{t+s})$, by optimizing over $(\bX, \by)$. The expectation of the uniform distribution $U$ is adopted to take into account all the segments along the
trajectory. $d(\cdot, \cdot)$ is a general distance measure, which can take the form of euclidean distance or cosine distance, etc.
During the optimization, we treat both $(\bX, \by)$ to be learnable variables, where $\bX$ is initialized with random noise, and $\by$ is initialized with equal probabilities over all labels. The detailed algorithm is shown in Algorithm \ref{alg:serverdistill}.
\begin{algorithm}[h]
\small
\renewcommand\arraystretch{0.8}
\caption{DataSyn}\label{alg:serverdistill}
\begin{algorithmic}[1]
\REQUIRE Global trajectory $\{\bw^c\}^L_1$, learning rate $\eta$, training rounds $N$, inner steps $s'$.
\STATE Randomly initialize $\bX_0$ and pair them with $\by$ to form the synthetic dataset.
\FOR{training iteration $n = 1, 2 \ldots N$}
\STATE Sample $t\sim U(1, L-s)$, then take $\bw^\text{t}$ and $\bw^\text{t+s}$ from the trajectory. 
\STATE Get the trained paramters $\Tilde{\bw} = \cA(\bX_n, \by_n, \bw^\text{t}, s')$
\STATE Calculate the distance $d(\Tilde{\bw},\bw^\text{t+s})$ and obtain gradient w.r.t $\bX_n$ and $\by_n$ as $\nabla_{\bX_n}d(\Tilde{\bw},\bw^\text{t+s})$ and $\nabla_{\by_n}d(\Tilde{\bw},\bw^\text{t+s})$.
\STATE Update $\bX_n$ and $\by_n$ using gradient descent $\bX_{n+1} = \bX_n - \eta\nabla_{\bX_n}d(\Tilde{\bw},\bw^\text{t+s})$, $\by_{n+1} = \by_n - \eta\nabla_{\by_n}d(\Tilde{\bw},\bw^\text{t+s})$
\ENDFOR

\textbf{Output:} Optimized synthetic data $\cD_\text{syn}$.
\end{algorithmic}
\end{algorithm}

Note that since the size of $\cD_\text{syn}$ is much
smaller than $\cD$, the effective step size should have a different scale. Therefore, in Equation \ref{eqn:inner_loop}, we may increase the number of steps from $s$ to $s'$ to account for this scale mismatch. 

In this way, the knowledge hidden in the dynamics of the global model can be transferred to $\cD_\text{syn}$, which then acts as an approximation of the global data distribution to aid the server-side aggregation.
\subsection{Overall Algorithm of DynaFed}
In this section, we present our \ourmodel that integrates the data synthesis process into the federated learning framework. 

Firstly, to construct the global trajectory on the server, we collect the global model's checkpoints from the first few communication rounds of FedAvg mainly considering the following two factors: 1) collecting a long trajectory induces prohibitive cost due to the expensive global communication, 2) the change in global model's parameters becomes insignificant during late rounds, which makes it difficult to extract knowledge from the dynamics.

Note that although $\cD_\text{syn}$ contains rich global knowledge, it is not sufficient to replace the global data distribution $\cD$ due to the following reasons: 1) the objective in Eqn.\ref{eqn:outer_loop} can not be ideally solved due to the two-level optimization procedure, 2) to make the scale of the optimization problem acceptable, the size of  $\cD_\text{syn}$ can not be as large as $\cD$, 3) there exists an inconsistency between the trainer $\cA(\cdot)$ and the one that produces the global trajectory, i.e., FedAvg. Therefore, instead of simply conduct regular training with $\cD_\text{syn}$, we leverage such $\cD_\text{syn}$ to help aggregation by finetuning the global model on the server side.

The rundown of the entire algorithm is presented in Algorithm \ref{alg:full_alg}. Firstly, the global model's trajectory in the earliest $L$ rounds is collected and stored on the server. Then, the server executes the data synthesis procedure to generate the pseudo data $\cD_\text{syn}$. In all subsequent rounds, $\cD_\text{syn}$ is leveraged on the server to help reduce the negative impact of deflected client models by finetuning the global model.

\begin{algorithm}[t]
\small
\renewcommand\arraystretch{0.8}
\caption{\ourmodel}\label{alg:full_alg}
\begin{algorithmic}[1]
\REQUIRE Client data ${\mathcal{D}}_m = \{(\bx_i, y_i)\}_{i=1}^{|\mathcal{D}_m|}$, global parameters $\bw$, client parameters $\{\bw_m\}^{|\mathcal{M}|}_{m=1}$. Total communication rounds $C$, 
local steps $T$,
data learning rate $\eta$,  
data training rounds $N$,  trajectory length $L$, inner steps for data synthesis $s'$.
\FOR{communication round $c = 1, 2 \ldots C$}
\STATE Sample active clients $\cM^c$ randomly. Distribute $\bw^c$ to active clients and initialize their parameters.
\FOR{all users $m \in \cM^c$}
\STATE $\{\bw_m^{c}\}^{|\cM^c|}_{m=1}$ $\leftarrow$ LocalTrain($\bw^{c}$, $T$)
\ENDFOR
\STATE Clients send updated $\{\bw^c_m\}^{|\mathcal{M}^c|}_{m=1}$ to server.
\STATE $\bw^{c+1}\leftarrow\frac{1}{|\cM^c|}\sum_{m\in\cM^c}\bw_m^{c}$
\IF{$c=L$}
\STATE $\cD_\text{syn}$ $\leftarrow$ DataSyn($\{\bw^c\}^L_{c=1}$, $\eta$, $s'$, $N$)
\STATE \ELSIF{$c>L$}
\STATE $\bw^{c+1}$ $\leftarrow$ Finetune($\cD_\text{syn}$, $\bw^{c+1}$)
\ELSE 
\STATE Add $\bw^{c+1}$ into trajectory list.
\ENDIF 
\ENDFOR

\end{algorithmic}
\end{algorithm}


Our method enjoys the following appealing properties: 
\begin{itemize}
\setlength{\itemsep}{0pt}
\setlength{\parsep}{0pt}
\setlength{\parskip}{0pt}
    \item In contrast to data-dependent KD methods \cite{lin2020ensembleFedDF, chen2020fedbe}, \ourmodel extracts the knowledge of the global data distribution from the global model trajectory, which does not depend on any external datasets;
    \item Compared with data-free KD methods \cite{zhu2021data, zhang2022fine}, \ourmodel is able to synthesize informative pseudo data in the early rounds of federated learning without requiring the global model to be well trained, which significantly helps convergence and stabilizes training;
    \item The data synthesis only needs to be conducted once. Besides, since only a few samples are synthesized, refining the global model on the server requires negligible time. 
\end{itemize}

\begin{remark}
We emphasize that our \ourmodel does not raise privacy concerns given following reasons: 1) we rely on only the trajectory of the global model's parameters, which is available in the conventional setting of all FedAvg-based methods; 2) we keep $\cD_{syn}$ at the server rather than sending it to the clients; 3) we extract global knowledge using the global dynamics, rather than any client-specific information as in~\cite{yin2018byzantine, huang2021evaluating,hatamizadeh2022gradvit, jeon2021gradient,fowl2021robbing}, the derived $\cD_\text{syn}$ comprises of information mixed with the entire global data distribution, thus prevents leakage of client privacy; 4) our \ourmodel shares a similar flavor as dataset condensation (DC) methods, which aims to generate informative pseudo data containing global knowledge, rather than real-looking data. The privacy-preserving ability of DC was also discussed in previous work \cite{dong2022privacy}.
\end{remark}

\section{Theoretical Analysis}
\begin{table*}[t!]
\small
    \centering
 \vspace{-0.25cm}
\scalebox{0.9}{
\begin{tabular}{c!{\vrule width 0.5pt}ccc!{\vrule width 0.5pt}ccc!{\vrule width 0.5pt}ccc}

\toprule
 &        \multicolumn{3}{c}{$\alpha=0.01$} & \multicolumn{3}{c}{$\alpha=0.04$} & \multicolumn{3}{c}{$\alpha=0.16$}\\
Method & FMNIST &  CIFAR10 &CINIC10 &    FMNIST &  CIFAR10 &CINIC10 &  FMNIST &  CIFAR10 &CINIC10\\
\midrule
FedAVG     & 74.50$\pm$1.32  &   39.30$\pm$3.42 &  31.60$\pm$5.50 & 81.74$\pm$1.98 &   51.19$\pm$2.85    &  45.35$\pm$3.00&  89.54$\pm$1.51 & 69.74$\pm$1.29& 55.40$\pm$2.05\\
FedProx    &   76.88$\pm$1.83 &  42.13$\pm$3.64 &  32.56$\pm$4.59&   83.06$\pm$2.53   & 58.93$\pm$2.14 &  46.30$\pm$2.87& 89.53$\pm$1.13 &  70.20$\pm$0.74 &  57.78$\pm$2.08\\
Scaffold    & 77.92$\pm$0.87  &   42.04$\pm$2.26 &  34.90$\pm$3.34 & 82.25$\pm$1.35 &   54.23$\pm$1.90    &  46.22$\pm$2.18&  88.54$\pm$0.32 & 68.57$\pm$0.91& 54.30$\pm$0.83\\
\midrule
FedDF$^*$     &  72.36$\pm$2.08 &  39.73$\pm$3.98 &  31.97$\pm$4.31& 81.65$\pm$0.97     & 54.20$\pm$2.93 &  45.79$\pm$2.95& 89.70$\pm$0.97& 70.71$\pm$0.94&  55.78$\pm$1.02\\
FedBE$^*$    &  72.33$\pm$1.79 &  38.36$\pm$3.74&  32.04$\pm$3.73 & 81.31$\pm$1.25     & 53.49$\pm$2.36 &  45.50$\pm$2.88& 89.62$\pm$0.75& 70.23$\pm$0.76&  55.42$\pm$1.37\\
ABAVG$^*$    &  75.98$\pm$1.99 &  39.95$\pm$1.37&  32.75$\pm$4.18 & 84.88$\pm$1.84     & 57.25$\pm$3.42 &  47.39$\pm$3.36 & 89.53$\pm$1.12& 70.55$\pm$2.41&  56.02$\pm$1.49\\
\midrule
FedGen$^\dagger$    &  75.59$\pm$1.12 &  40.19$\pm$2.14 &  32.59$\pm$3.25& 81.46$\pm$1.08     & 56.60$\pm$1.29&  45.57$\pm$2.70 & 89.95$\pm$0.89& 70.89$\pm$0.54&  55.34$\pm$1.13\\
FedMix$^\dagger$    & 81.34$\pm$0.68  &   50.48$\pm$1.23 &  37.15$\pm$1.81 & 84.23$\pm$0.50 &   62.77$\pm$1.07    &  50.22$\pm$1.41&  89.05$\pm$0.24 & 70.33$\pm$0.55& 56.74$\pm$0.45\\
\textbf{DynaFed$^\dagger$}    &   \textbf{87.52$\pm$0.15}  &   \textbf{65.53$\pm$0.34}&  \textbf{48.04$\pm$0.70} & \textbf{89.45$\pm$0.11} &   \textbf{70.07$\pm$0.12}      &  \textbf{55.43$\pm$0.24}&  \textbf{91.35$\pm$0.07}& \textbf{74.69$\pm$0.14}&  \textbf{59.80$\pm$0.10}\\
\bottomrule
\end{tabular}
}
\caption{Comparison of test performances achieved by different FL methods with different degrees of data heterogeneity $\alpha$ across multiple datasets. We report the mean test accuracy of last five communication rounds. $^*$Methods assume the availability of proxy data. $^\dagger$ Methods are based on data sharing or generation. We observe that our approach outperforms other methods by a large margin, and its advantage is more prominent on more challenging datasets with higher heterogeneity. Specifically, \ourmodel demonstrates relative improvement over the FedAvg baseline by 17.5\%, 64.5\%, 52.0\%, and 82.2\% on FMNIST, CIFAR10, CINIC10, and CIFAR100, respectively.}
\label{tab:main_exp}
\vspace{-0.35cm}
\end{table*}

In this section, we present some insights to understand our data synthesis process from the neural tangent kernel (NTK) theory and also give the convergence results for \ourmodel. Detailed proofs are given in the appendix. 

In our data synthesis process, we align the segments (e.g., from $t$ to $t+s$ ) of the trajectories trained on $\cD$ and $\cD_\text{syn}$ for any $t$. Essentially, those segments are the cumulative gradients calculated using $\cD$ and $\cD_\text{syn}$, which approximate to $\nabla \cL(\bw_t, \cD_{\text{syn}} )\Delta t$ and $\nabla \cL(\bw^t, \cD )\Delta t$. Therefore, we can expect $\nabla \cL(\bw, \cD_{\text{syn}} )$ would be close to $\nabla \cL(\bw, \cD )$, which is verified by our experimental result (left of Figure.\ref{fig:data_quality}). Given a sample $\bx$, if continuous-time gradient descent is adopted as the training solver, the dynamics of neural network function trained on $\cD_\text{syn}$ and $\cD$ take the forms of 
\begin{align}
\begin{cases}
    &\frac{d{f}(\bx,{\bw}^t)}{dt} = \nabla_{{\bw}^t}{f}(\bx,{\bw}^t)^\top\nabla  \cL({\bw}^t, \cD_{\text{syn}}),\\
     &\frac{df(\bx,\bw^t)}{dt} = \nabla_{\bw^t}f(\bx,\bw^t)^\top\nabla  \cL(\bw^t, \cD).
\end{cases}\label{eqn:ode}
\end{align}
which is close to an ordinary differential equation according to the NTK theory\cite{jacot2018neural}. 
Note that the right-hand sides of the two equations in (\ref{eqn:ode}) are close if $\nabla \cL(\bw, \cD_{\text{syn}} )\approx \nabla \cL(\bw, \cD )$.  In this case, the following lemma about the continuous dependence of differentiable equation shows that the neural functions $f(\bx, \bw^t)$ learned with pseudo data $\cD_\text{syn}$ are similar to that learned with real global data $\cD$ during the whole training process. This further indicates that $\cD_\text{syn}$ achieves similar effect as $\cD$ during training.

\begin{lemma}[Continuous Dependence\cite{wolfgang1998ordinary}] Suppose $\tilde{F}(t,f)$ and $F(t,f)$ are two continuous functions in a region $G$ satisfying 
$$|\tilde{F}(t,f)-F(t,f)|\leq \epsilon, \forall (t,f) \in G.$$
Further, we assume $F(t,f)$ satisfy the $L_F$-Lipschitz condition w.r.t., $f$.
Let $\tilde{f}(t)$ and $f(t)$ be the solutions of initial problems, $$\frac{d\tilde{f}}{dt}=\tilde{F}(t,\tilde{f}) \mbox{ and } \frac{df}{dt}=F(t,f),$$
with $\tilde{f}(t_0) = f_0$ and $f(t_0) = f_0$. Then, in a common region $|t-t_0|\leq \alpha$, we have the following estimation:
$$|\tilde{f}(t)-f(t)| \leq \frac{\epsilon}{L_F}\left( e^{\alpha L_F}-1\right).$$
\end{lemma}

To analyze the convergence of \ourmodel, we need to define some additional notations  and rewrite our method as follows. Suppose with the $\cD_{\text{syn}}$ generated from the data synthesis process,  in \ourmodel we run SGD for $\tau_1$ and $\tau_2$ iterations in each local training round and finetuning process, respectively. Let the  sets $\cI$ and $\cJ$ be 
\begin{align}
  &\cI = \{t| t = k(\tau_1 + \tau_2)+ c, k =0,1,2,\ldots, c \in [\tau_1]\},\nonumber \\
  &\cJ = \{t| t = k(\tau_1 + \tau_2)+ \tau_1, k =0,1,2,\ldots\},
\end{align}
where $[\tau_1]= \{0,1,\ldots, \tau_1-1\}$. Therefore, when $t\in \cI$, we perform local training, while when $t\notin \cI$, we conduct finetuning. $\cJ$ denotes the time index for aggregation. The detailed steps of \ourmodel can be rewritten as 
\begin{align}
\small
    &\bv_{t+1}^m= \begin{cases}
    \bw_m^t - \eta_t\nabla  \ell(\bw_m^t,\xi_m^t), &\mbox{if } t \in \cI\\
    \bw_m^t - \eta_t\nabla  \cL(\bw_m^t, \cD_{\text{syn}}),& \mbox{if } t\notin \cI 
    \end{cases},\nonumber\\
    &\bw_m^{t+1} = \begin{cases}
    \bv_m^{t+1}, & \mbox{ if } t+ 1 \notin \cJ\\
    \sum_{m=1}^M \alpha_m \bv_m^{t+1}, &\mbox{ if } t+1 \in \cJ
    \end{cases},\nonumber
\end{align}
where $\xi_m^t  \sim \cD_m$. Based on the notations, we define a sequence: 
\begin{small}
{
\setlength\abovedisplayskip{0mm}
\setlength\belowdisplayskip{1mm}
\begin{align}
\small
    \bar{\bw}^{t} = \sum_{m=1}^M \alpha_m \bw_m^t.\nonumber
\end{align}
}\end{small}Hence, our algorithm is an integration of FedAvg and a biased GD. For $\bar{\bw}^t$, note that $\bar{\bw}^t =\bw_1^t=\cdots=\bw_M^t$ in the finetuning process and we have the following convergence results:
\begin{theorem}[Convergence]  For $\tilde{L}$-smooth, $\mu$-strongly convex loss functions $\ell(\cdot, \cdot)$. We assume $\| \nabla\cL(\bw, \cD_{\text{syn}}) - \nabla \cL(\bw, \cD)\| \leq \delta \|\nabla \cL(\bw, \cD)\| + \epsilon$ holds with two small non-negative scalars $\delta$ and $\epsilon$. Let $\eta_t = \frac{c}{t}$ for a proper constant $c$. Then, \ourmodel satisfies 
\begin{align}
\small
    \mathbb{E}\cL( \bar{\bw}^T, \cD) - \cL( {\bw}^*, \cD) \leq \frac{ C}{T},
\end{align}
where  $\bw^*$ is the minimum of $\cL(\bw, \cD)$ and  $C$ is a constant, whose detailed formula is given in the appendix.
\end{theorem}
 More detailed discussions about the convergence result are given in the Appendix.

\section{Experiments}
\paragraph{Benchmark Datasets and Experimental Settings.}
\begin{figure*}[t!]
    \centering
    \begin{subfigure}[b]{0.24\textwidth}
        \centering
        \includegraphics[width=\textwidth]{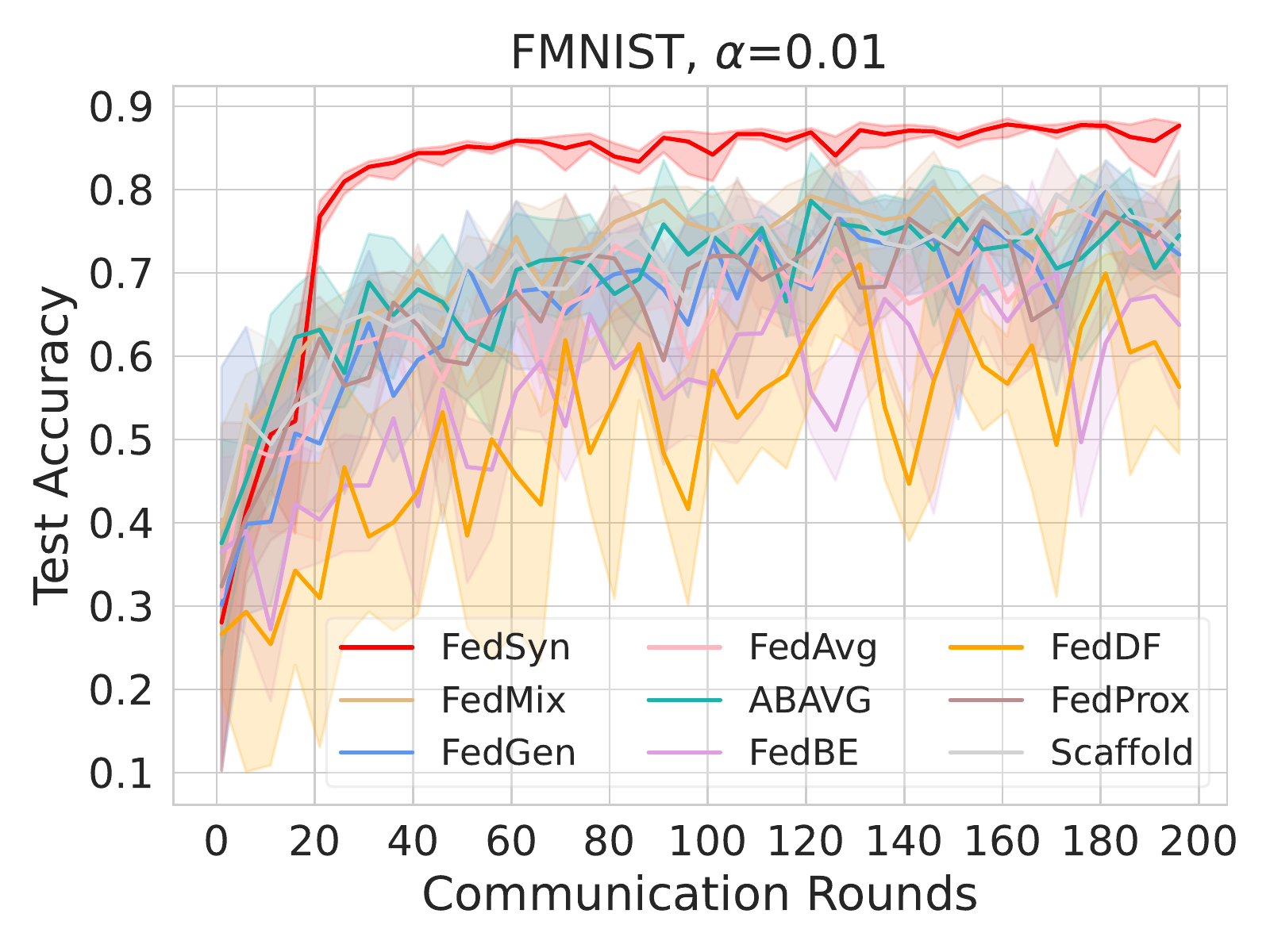}
        \label{fig:mean and std of net14}
    \end{subfigure}
    \begin{subfigure}[b]{0.24\textwidth}  
        \centering 
        \includegraphics[width=\textwidth]{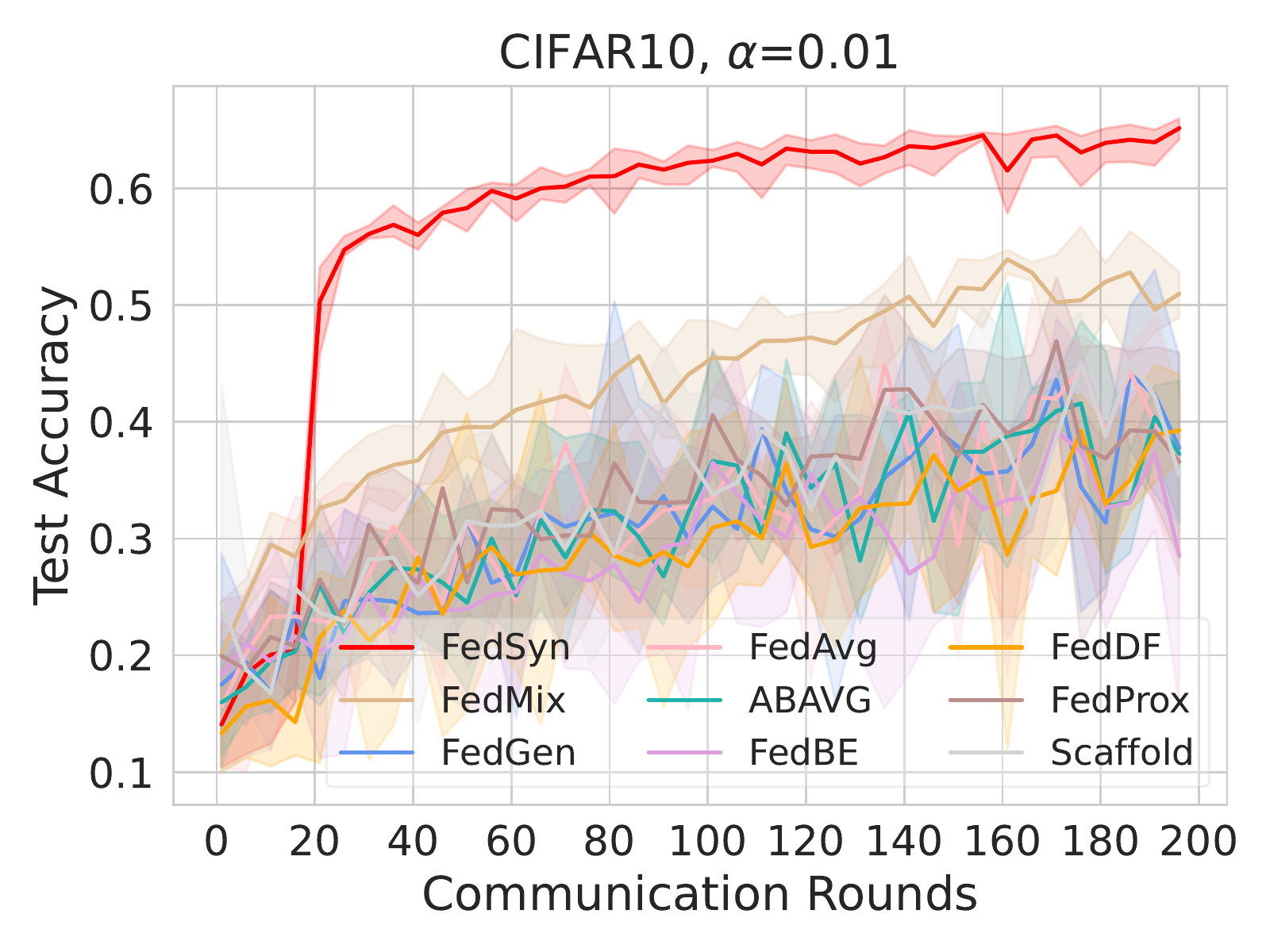}
        \label{fig:mean and std of net24}
    \end{subfigure}
    \begin{subfigure}[b]{0.24\textwidth}   
        \centering 
        \includegraphics[width=\textwidth]{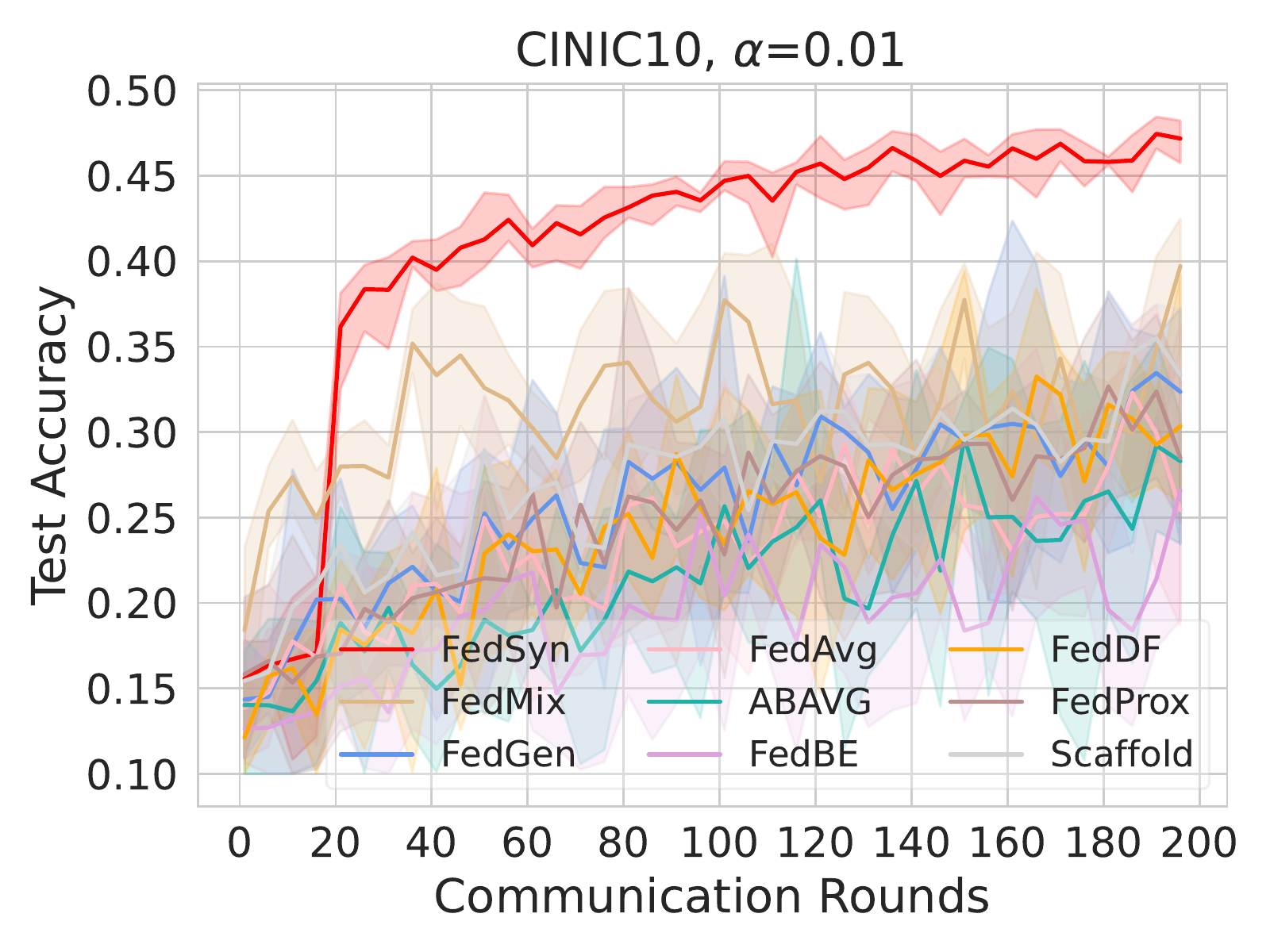}
        \label{fig:mean and std of net34}
    \end{subfigure}
    \begin{subfigure}[b]{0.24\textwidth}   
        \centering 
        \includegraphics[width=\textwidth]{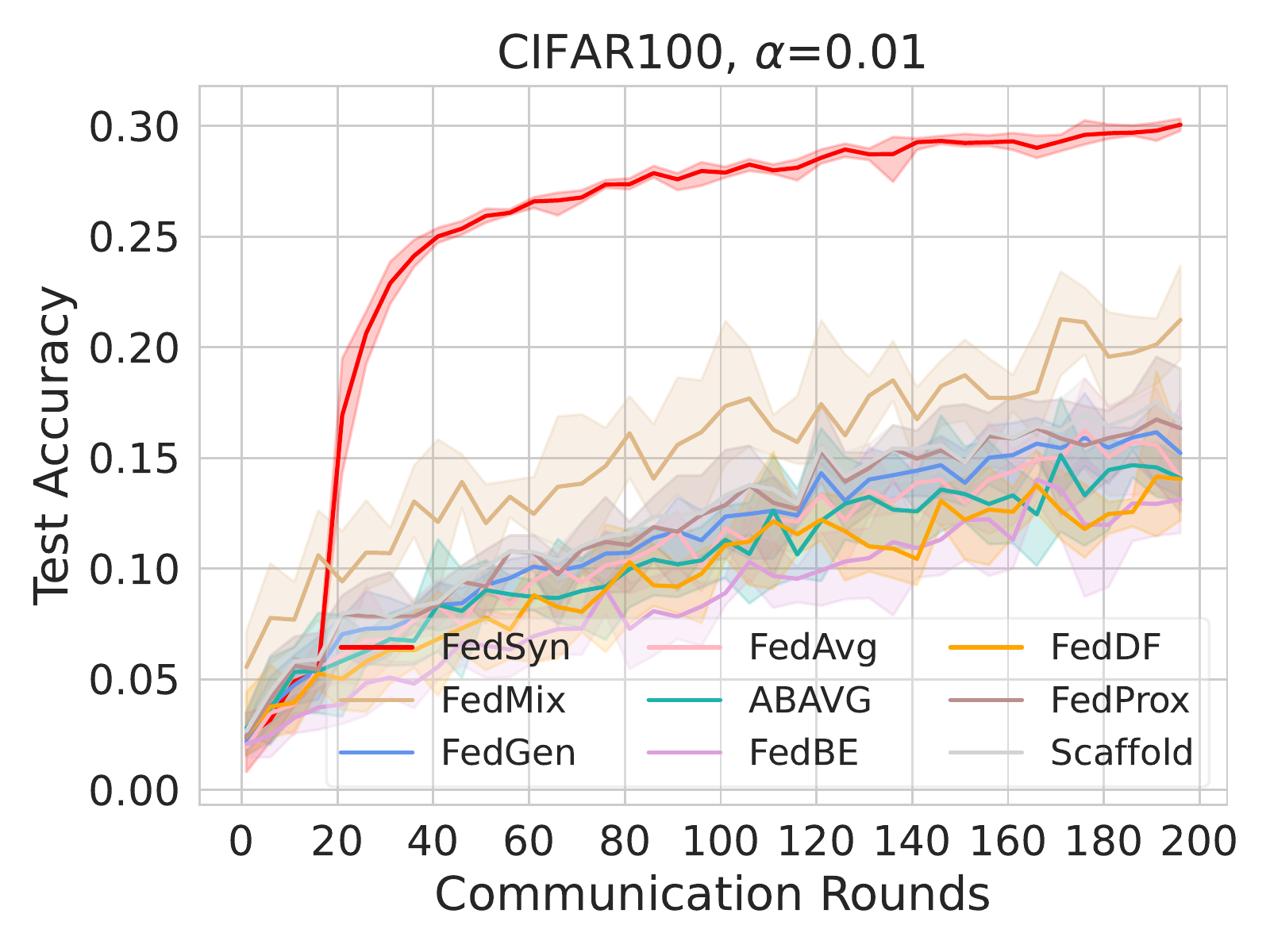}
        \label{fig:mean and std of net34}
    \end{subfigure}
    \vskip -0.49cm
    \begin{subfigure}[b]{0.24\textwidth}   
        \centering 
        \includegraphics[width=\textwidth]{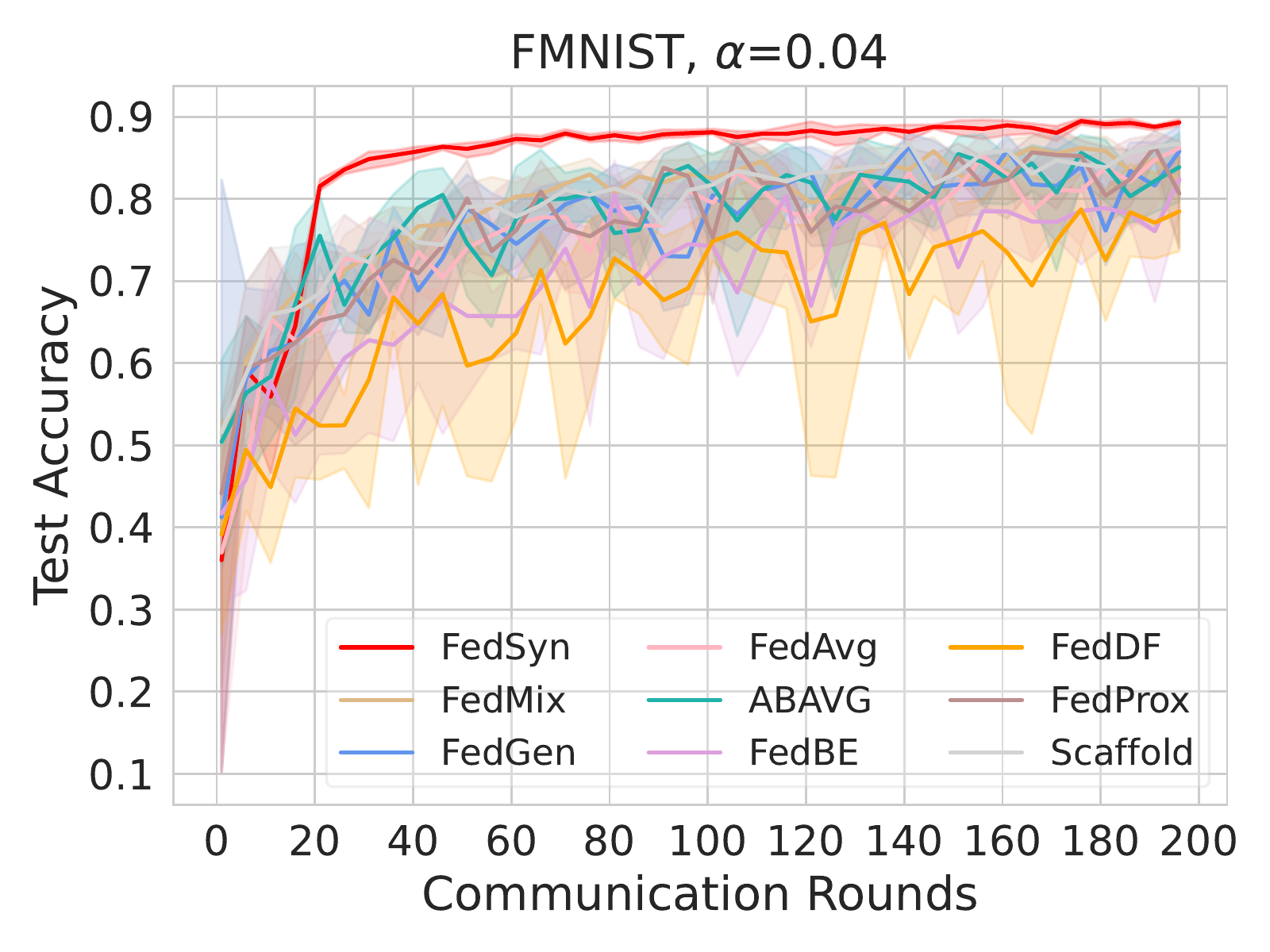}
        \label{fig:mean and std of net44}
    \end{subfigure}
    \begin{subfigure}[b]{0.24\textwidth}   
        \centering 
        \includegraphics[width=\textwidth]{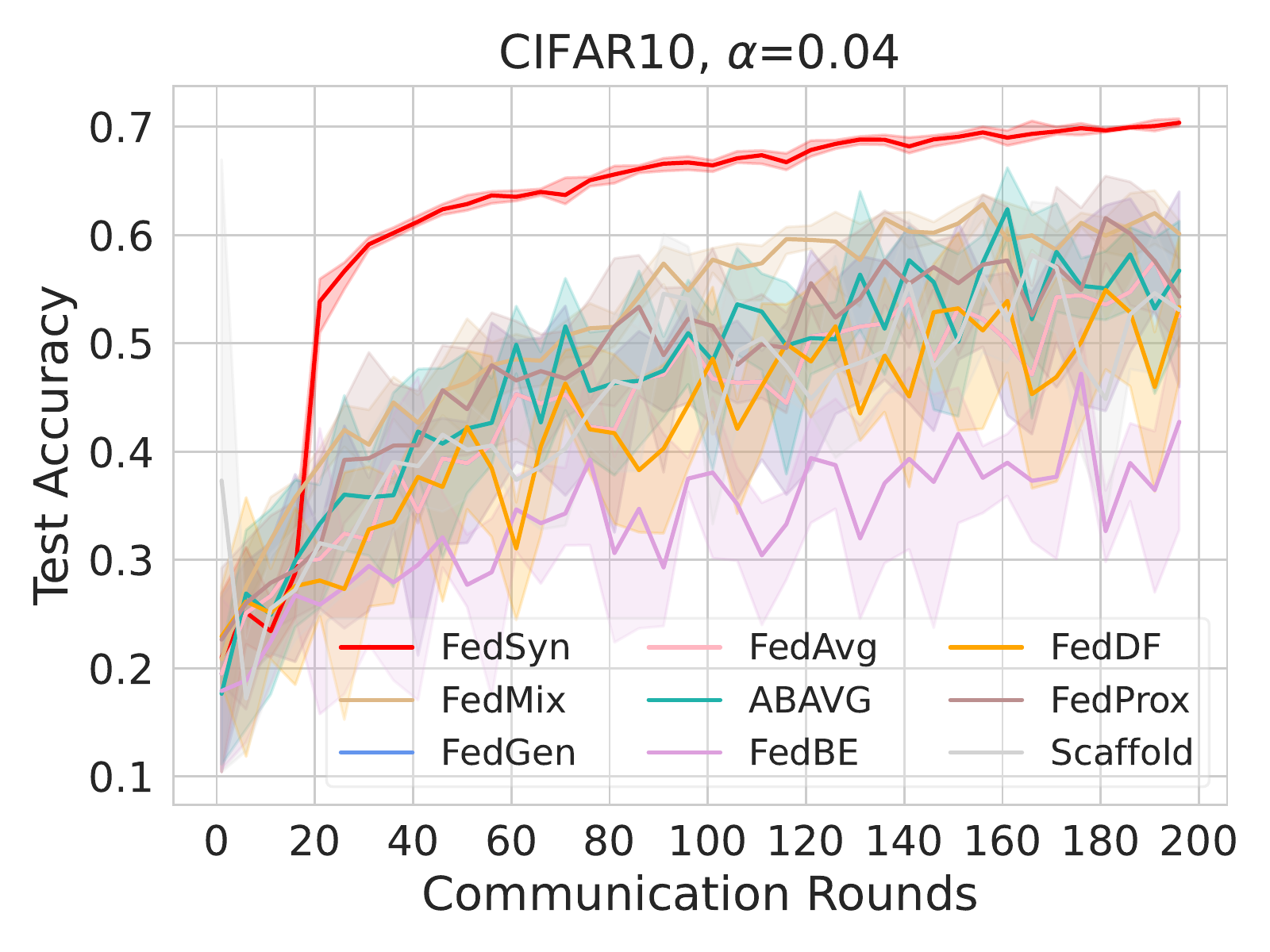}
        \label{fig:mean and std of net44}
    \end{subfigure}
    \begin{subfigure}[b]{0.24\textwidth}   
        \centering 
        \includegraphics[width=\textwidth]{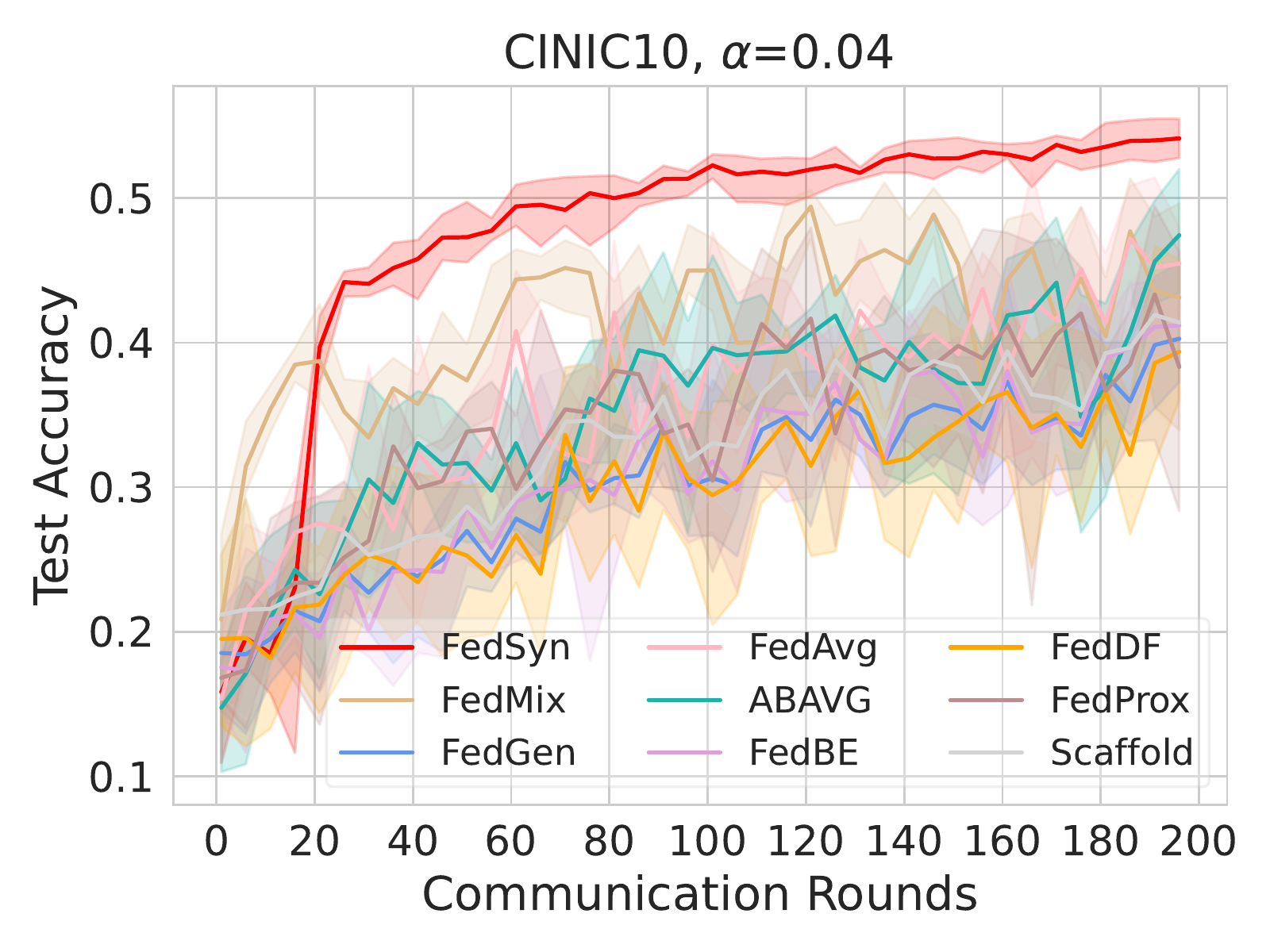}
        \label{fig:mean and std of net44}
    \end{subfigure}
    \begin{subfigure}[b]{0.24\textwidth}   
        \centering \includegraphics[width=\textwidth]{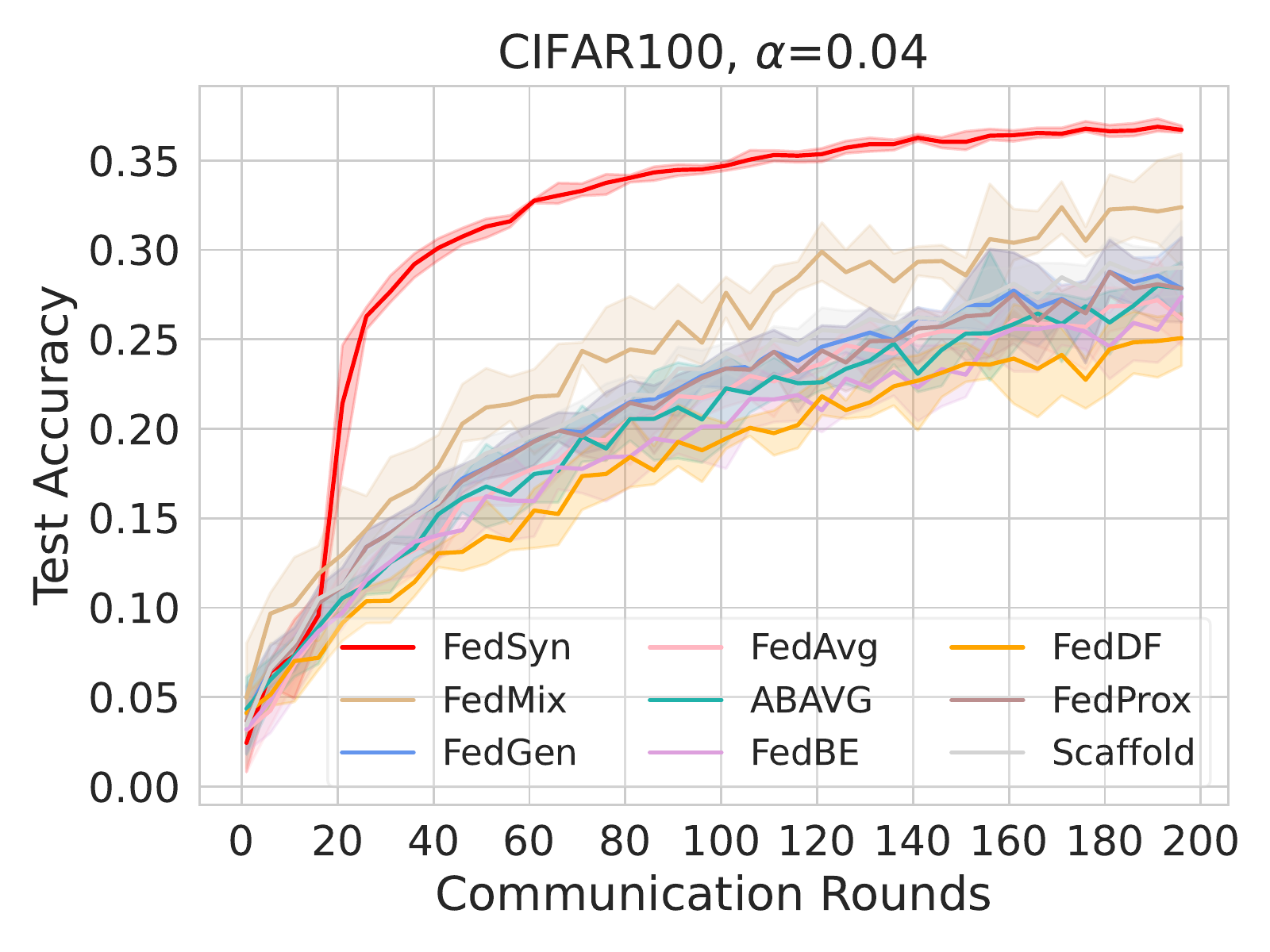}
        \label{fig:mean and std of net44}
    \end{subfigure}
    \vspace{-0.5cm}
    \caption{\small Visualization of global model's test performance on various datasets throughout the global communication rounds. We can see that the global model rapidly converges to a satisfactory test accuracy once $\cD_\text{syn}$ participates in refining the global model. Furthermore, $\cD_\text{syn}$ also helps reduce the fluctuation of model performances between communication rounds, which significantly boosts the training stability. \ourmodel requires less than 10\% communication rounds to achieve comparable performance with the baseline methods.} 
    \label{fig:process_curve}
\end{figure*}
 \begin{table}[t!]
\tiny
    \centering
 \vspace{-0.25cm}
\begin{scriptsize}
\begin{tabular}{ccccc}

\toprule
Method &   $\alpha=0.01$ & $\alpha=0.04$& $\alpha=0.08$& $\alpha=0.16$ \\
\midrule
FedAVG     &   16.54$\pm$2.18 & 26.56$\pm$1.53& 34.54$\pm$1.02&   39.65$\pm$0.94 \\
FedProx   &  18.46$\pm$1.05  & 28.58$\pm$1.46 & 34.82$\pm$0.54&  40.98$\pm$0.49 \\
Scaffold    &   17.33$\pm$1.21& 28.46$\pm$1.18& 35.04$\pm$0.35&   40.57$\pm$0.33 \\
\midrule
FedDF$^*$     & 16.02$\pm$1.94  & 26.94$\pm$1.25&  34.77$\pm$0.88&   39.76$\pm$0.44 \\
FedBE$^*$   & 15.78$\pm$2.34  & 28.03$\pm$0.34&   33.91$\pm$0.79 & 39.45$\pm$0.79\\
ABAVG    &  16.52$\pm$1.98 & 29.14$\pm$0.57&   34.66$\pm$0.98   & 41.00$\pm$0.23\\
\midrule
FedGen$^\dagger$   &  16.51 $\pm$1.32& 27.03$\pm$1.14&   34.56$\pm$0.78 & 39.96$\pm$0.58\\
FedMix$^\dagger$   &  23.54$\pm$0.96 & 32.18$\pm$0.59&  36.30$\pm$0.42 & 41.09$\pm$0.14\\
\textbf{DynaFed$^\dagger$}    &   \textbf{30.14$\pm$0.19} & \textbf{36.79$\pm$0.12}   & \textbf{40.02$\pm$0.09} & \textbf{42.47$\pm$0.06} \\
\bottomrule
\vspace{-5mm}
\end{tabular}
\caption{Comparison of test performances on CIFAR100 with different degrees of data heterogeneity $\alpha$.}\label{tab:cifar100_exp}
\vspace{-5mm}
\end{scriptsize}
\end{table}

We conduct experiments over four commonly used datasets: FashionMNIST\cite{xiao2017/online}, CIFAR10\cite{krizhevsky2009learning}, CINIC10\cite{darlow2018cinic} and CIFAR100\cite{krizhevsky2009learning}. Among them, FashionMNIST is a dataset containing grey-scale images of fashion products. CIFAR10 is an image classification dataset containing daily objects. CINIC10 is a dataset combining CIFAR10 and samples from similar classes that are downsampled from ImageNet\cite{krizhevsky2017imagenet}. These three datasets contain 10 classes. CIFAR100 contains the same data as CIFAR10, but categorizes the data into 100 classes. For each dataset, we mainly conduct experiments with heterogeneous client data distribution. We follow prior work \cite{lin2020ensembleFedDF, chen2020fedbe} to use Dirichlet distribution for simulating the non-IID data distribution, where the degree of heterogeneity is defined by $\alpha$, smaller $\alpha$ value corresponds to more severe heterogeneity. 

\vspace{-1em}
\paragraph{Baseline Methods}
We consider various state-of-the-art solutions against non-IID data distribution in the context of federated learning. Specifically, we compare with the following approaches 1) the vanilla aggregation strategy FedAVG~\cite{mcmahan2017communication}; 2) regularization-based strategies FedProx~\cite{li2020federated}, Scaffold~\cite{karimireddy2019scaffold}; 3) data-dependent knowledge distillation strategies that need external dataset  FedDF~\cite{lin2020ensembleFedDF} and FedBE~\cite{chen2020fedbe}, ABAVG~\cite{xiao2021novel}; (4) data sharing \cite{yoon2021fedmix} or data-free knowledge distillation \cite{zhu2021data} methods. Note that we do not compare with \cite{zhang2022fine} since the code is not published. Please refer to Appendix for detailed settings of the baseline methods.

\vspace{-0.5em}
\paragraph{Configurations} Unless specified otherwise, we follow \cite{Fedaux, chen2021bridging, xie2022optimizing} and adopt the following default configurations throughout the experiments: we run 200 global communication rounds with local epoch set to 1. There are 80 clients in total, and the participation ratio in each round is set to 40\%. Experiments using other participation ratios are in the Appendix. 
We report the global model's average performance in the last five rounds  evaluated using the test split of the datasets. For the construction of global trajectory, we first run FedAvg~\cite{mcmahan2017communication} and use the checkpoints from the first 20 communication rounds ($L=20$). 
We set the time difference $s$ between the start and end checkpoint to 5, and the target checkpoint is averaged with 2 checkpoints sampled between $\w^\text{t}$ and $\w^\text{t+s}$. More details can be found in the Appendix.

\subsection{Main Experiments with Data Heterogeneity} We demonstrate the superior performance of our \ourmodel  by conducting experiments on heterogeneous client data across comprehensive datasets and various heterogeneity values $\alpha$. 
Specifically, we use three datasets with 10 classes (shown in Table \ref{tab:main_exp}): FashionMNIST\cite{xiao2017/online}, CIFAR10\cite{krizhevsky2009learning} and CINIC10\cite{darlow2018cinic}, heterogeneity degree $\alpha$ set to 0.01, 0.04 and 0.16; and CIFAR100 containing 100 classes, with $\alpha$ values 0.01, 0.04 ,0.08 and 0.16 (shown in Table \ref{tab:cifar100_exp}). \ourmodel significantly boosts the convergence, stabilizes training, and brings considerable performance improvement compared with previous approaches. Specifically, with heterogeneity value $\alpha=0.01$, \ourmodel demonstrates relative improvement over the FedAvg baseline by 17.5\%, 64.5\%, 52.0\%, and 82.2\% on FMNIST, CIFAR10, CINIC10, and CIFAR100, respectively.

As demonstrated in Figure \ref{fig:process_curve}
, the performance of \ourmodel is rapidly boosted as soon as the synthesized data starts refining the global model on the server. This verifies that \ourmodel does not depend on the global model's performance in data synthesis, which is consistent with our analysis in Section \ref{sec:proposed_method}. This characteristic enables faster convergence to achieve good performance with fewer communication rounds. As shown in Figure \ref{fig:process_curve}, \ourmodel requires less than 20\% communication rounds to achieve comparable performance with the baseline methods.

\subsection{Detailed Analysis}
We conduct a detailed analysis of \ourmodel and aim to provide answers to the following questions: (1) Does $\cD_\text{syn}$ contain information about global data distribution while protecting client privacy? (2) Can we leverage just the dynamics of the early rounds to synthesize $\cD_\text{syn}$? (3) How many pseudo samples do we need to synthesize to ensure effectiveness? (4) Does $\cD_\text{syn}$ still help convergence under more severe heterogeneity and longer local training?

\begin{figure}[t]
        \centering
        \includegraphics[width=1.0\linewidth]{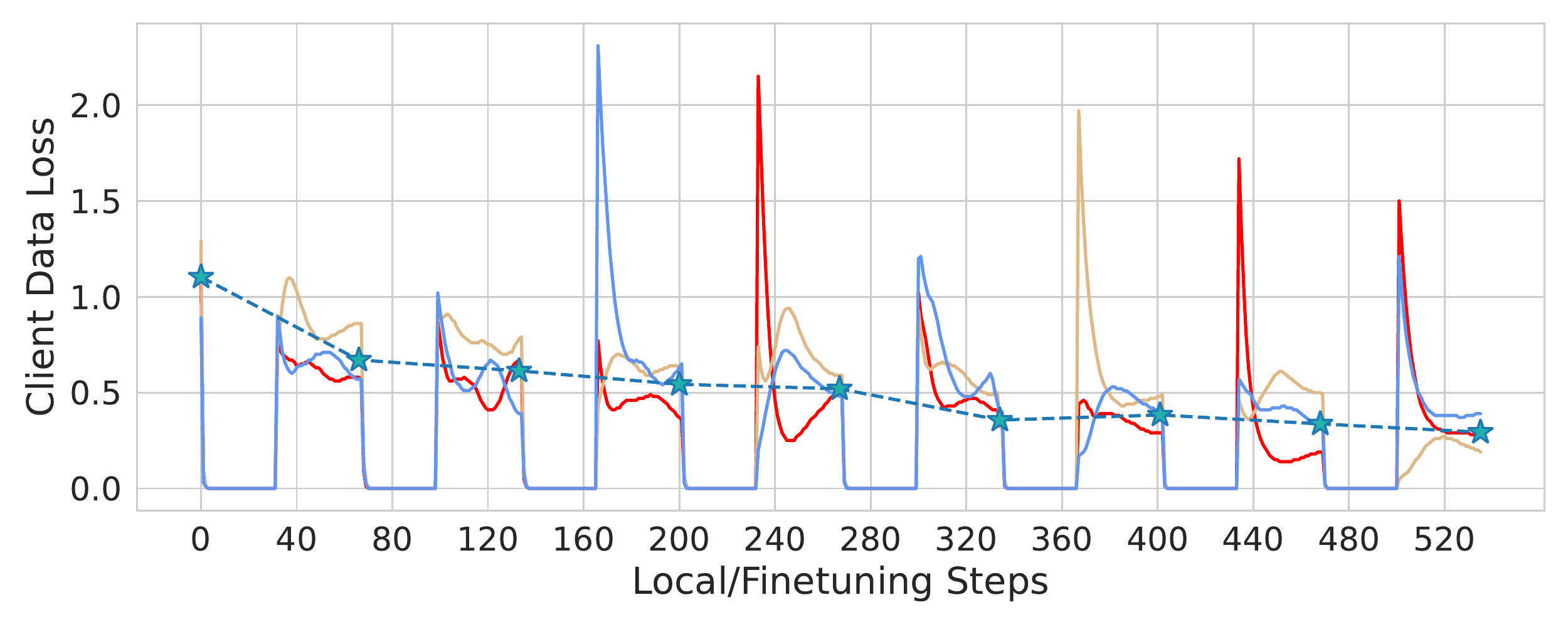}
\vspace{-6mm}
    \caption{Loss curves over each client's data throughout local training and finetuning. Each of the 3 colors represents the loss over one client's data. The stars are the global model's average losses over all client data after finetuning with $\cD_\text{syn}$. During local training, the client losses quickly converge to near zero. However, due to deflection caused by heterogeneity, the losses over some clients' data dramatically increase after aggregation. Finetuning with $\cD_\text{syn}$ decreases those losses and reduces the aggregation bias.}
     \vspace{-0.2cm}
    \label{fig:global_info}
\end{figure}
\begin{figure}
    \centering
    \begin{subfigure}[b]{0.11\textwidth}
        \centering
        \includegraphics[width=\textwidth]{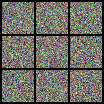}
        {{\small Iteration 0}}    
        \label{fig:mean and std of net14}
    \end{subfigure}
    \begin{subfigure}[b]{0.11\textwidth}
        \centering
        \includegraphics[width=\textwidth]{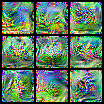}
        {{\small Iteration 125}}    
        \label{fig:mean and std of net14}
    \end{subfigure}
    \begin{subfigure}[b]{0.11\textwidth}   
        \centering 
        \includegraphics[width=\textwidth]{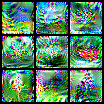}
        {{\small Iteration 250}}    
        \label{fig:mean and std of net34}
    \end{subfigure}
    \begin{subfigure}[b]{0.11\textwidth}
        \centering
        \includegraphics[width=\textwidth]{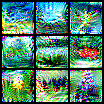}
        {{\small Iteration 500}}    
        \label{fig:mean and std of net14}
    \end{subfigure}
    \caption[ The average and standard deviation of critical parameters ]
    {\small Visualization of learned synthetic data on 3 classes from CIFAR10 throughout the optimization process. In the beginning, the pixels are randomly initialized and  contain little information. As the optimization goes on, some patterns emerge in the synthetic images but remain unrecognizable.} 
    \label{fig:privacy_of_dsyn}
    \vspace{-5mm}
\end{figure}

\vspace{-0.5em}
\paragraph{$\cD_\text{syn}$ Contains Global Information and Preserves Privacy.} We conduct experiment with CIFAR10 and set $\alpha=0.01$, where client datasets are extremely \textit{non-iid}. We track the losses calculated over each client's data throughout local training as well as the global model's finetuning. During local training, we calculate the client models' losses over their own datasets, i.e., $\cL_m(\bw_m, \cD_m)$. During finetuning, we calculate the global model's losses over each client's dataset, i.e., $\cL_m(\bw, \cD_m)$. To prevent cluttering, we randomly select 3 client datasets for illustration. The result is shown in Figure \ref{fig:global_info}, each color represents the loss over one clients' dataset. We observe that the client models easily overfit during local training due to the extreme class imbalance. The deflected client models make the aggregated global model demonstrate high loss values over some clients' data. Remarkably, we observe that finetuning with $\cD_\text{syn}$ is able to recover the global model to a reasonable state, which achieves small losses over all client datasets. This verifies that $\cD_\text{syn}$ contains information of the global data distribution. Furthermore, Figure \ref{fig:privacy_of_dsyn} presents the synthesized data of CIFAR10, client-specific information can not be observed.

\begin{figure}
    \centering
    \begin{subfigure}[b]{0.23\textwidth}
        \centering
        \includegraphics[width=\textwidth]{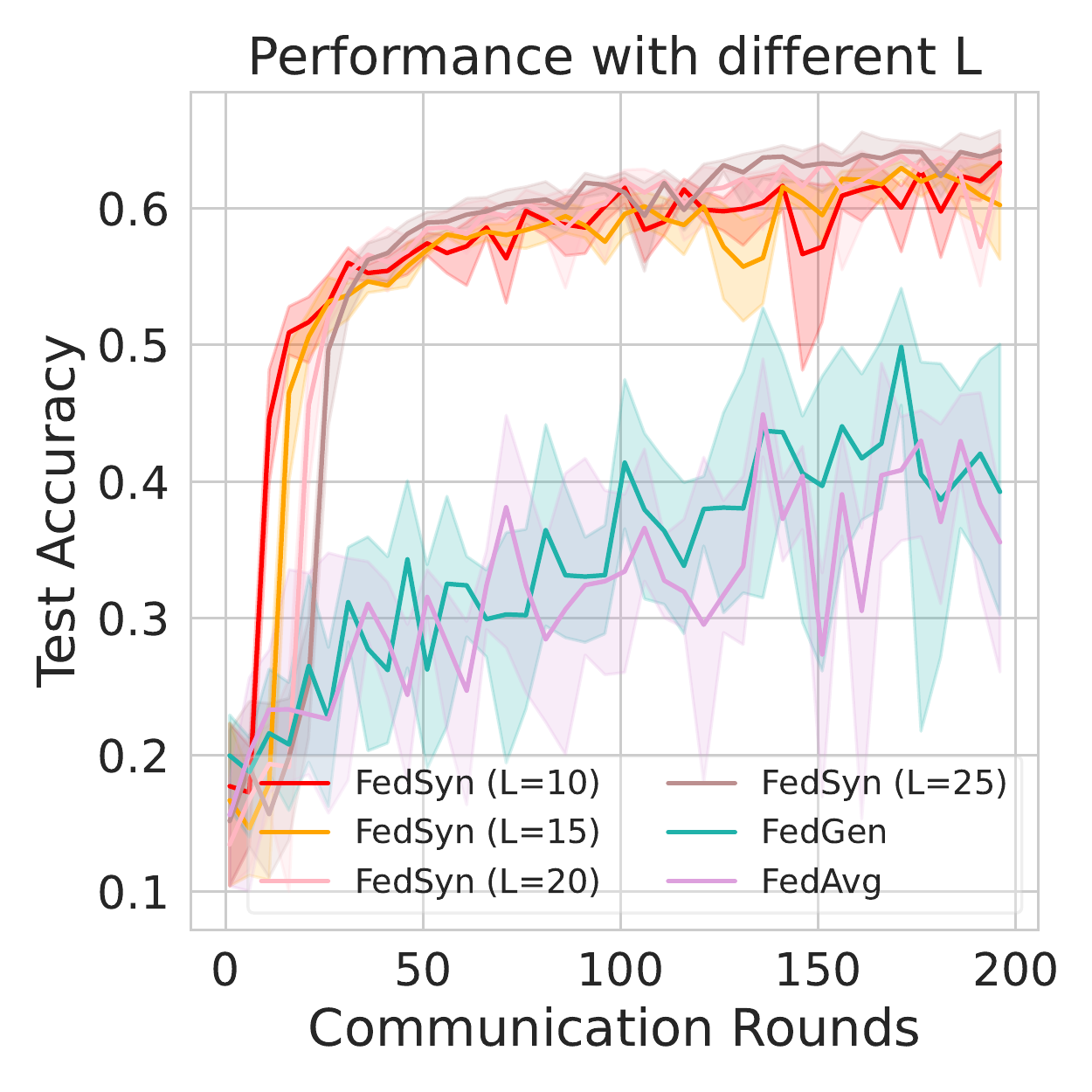}
    \end{subfigure}
    \hfill
    \begin{subfigure}[b]{0.23\textwidth}   
        \centering 
        \includegraphics[width=\textwidth]{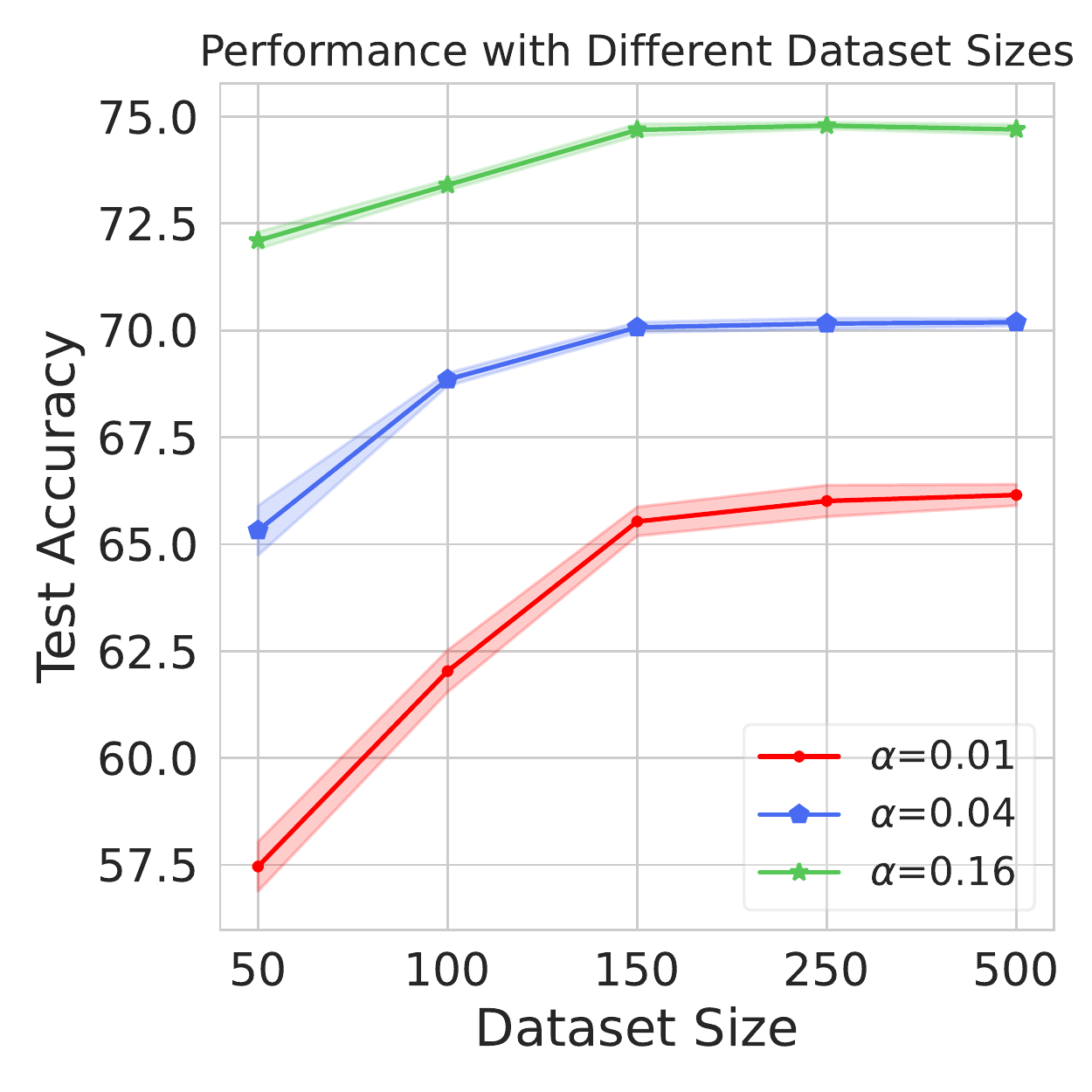}
    \end{subfigure}
    \hfill
    \caption{Left: The test accuracy curves for different choices of trajectory length L on the CIFAR10 dataset with $\alpha=0.01$. By leveraging the dynamics of the global model's trajectory in the first few rounds, e.g., $\text{L}\in \{10, 15, 20, 15\}$, the derived $\cD_\text{syn}$ already helps achieve faster convergence and stable training.  In contrast, the FedGen approach only brings slight performance gain during the later phase of training due to the dependence on the global model's performance. Right: We show the performance of \ourmodel with different sizes of $\cD_\text{syn}$  for various  $\alpha$. The performance gain is significant with just 150 synthesized samples.}
     \vspace{-0.35cm}
    \label{fig:length_size}
\end{figure}

\vspace{-0.5em}
\paragraph{$\cD_\text{syn}$ Can be Learned with Early Trajectory.}
We conduct experiments with different choices of trajectory length $L$ in left of Figure \ref{fig:length_size}. We observe that even if with $L=10$, the result is comparable with performance obtained with longer trajectory $L=25$. Compared with the baseline methods, \ourmodel achieves significant convergence speedup and performance boost. These results support our claim in Section \ref{sec:proposed_method} that our method can take effect early during training.
By contrast, the data-free KD method FedGen \cite{zhu2021data} that trains a generator to produce pseudo data starts to show a slight improvement only in the late stage of training since it depends explicitly on the global model's performance when training the generator.




\vspace{-0.5em}
\paragraph{How the Size of $\cD_\text{syn}$ Impacts the Performance.} As shown in the right of Figure \ref{fig:length_size}, we conduct experiments with different sizes of $\cD_\text{syn}$ and various heterogeneity degrees $\alpha$. We find that a small $\cD_\text{syn}$ suffices for good performance, while larger $\cD_\text{syn}$ brings only marginal performance boost. 
This property not only saves the cost for synthesizing $\cD_\text{syn}$, but also makes the finetuning of the global model more efficient. 

\vspace{-0.5em}
\paragraph{$\cD_\text{syn}$ is Able to Mimic the Global Dynamics.} In the left of Figure \ref{fig:data_quality}, we calculate the cosine distance between the target checkpoint $\bw^\text{t+s}$ and the parameters $\Tilde{\bw}$ trained from $\bw^\text{t}$ for $s'$ steps with $\cD_\text{syn}$, a randomly sampled real dataset of the same size as $\cD_\text{syn}$, and a dataset consisted of noisy pixels, respectively. We can see that in terms of mimicking the global trajectory, $\cD_\text{syn}$ not only significantly outperforms the noise dataset, but also achieves only half of the distance obtained with real dataset, which is not accessible in FL setting. This verifies the ability of $\cD_\text{syn}$ to mimic global trajectory.


 \begin{table}[t!]
 \begin{minipage}[c]{0.48\textwidth}
\tiny
    \centering
 \vspace{-0.25cm}
\begin{scriptsize}
\begin{tabular}{c!{\vrule width 0.5pt}cc!{\vrule width 0.5pt}cc}
\toprule
 &        \multicolumn{2}{c}{$\alpha=0.01$} & \multicolumn{2}{c}{$\alpha=0.04$}\\
Method &   5 epochs  & 10 epochs  &   5 epochs  & 10 epochs \\
\midrule
FedAVG     &   33.23$\pm$3.54 & 29.93$\pm$4.62 &   50.28$\pm$2.17    & 46.09$\pm$2.95  \\
FedProx   &  42.60$\pm$2.30 & 42.86$\pm$2.84 & 58.40$\pm$1.35& 54.30$\pm$1.98\\
Scaffold    &   39.43$\pm$1.86 & 36.52$\pm$2.04& 55.46$\pm$1.25 & 50.05$\pm$1.57\\
\midrule
FedDF$^*$     & 31.68$\pm$3.16  & 39.85$\pm$3.79&  52.31$\pm$2.38 & 50.90$\pm$2.53\\
FedBE$^*$   & 35.49$\pm$2.88  & 34.19$\pm$3.34& 49.78$\pm$1.79  & 51.34$\pm$1.90\\
ABAVG$^*$    &  37.87$\pm$2.57 & 35.08$\pm$3.03 & 56.81$\pm$1.94  & 52.17$\pm$2.32\\
\midrule
FedGen$^\dagger$   &  35.64$\pm$2.52 & 35.03$\pm$3.58&  57.60$\pm$1.55 & 54.48$\pm$2.03\\
FedMix$^\dagger$   &  47.36$\pm$1.24 & 41.53$\pm$1.37& 60.74$\pm$0.95  & 56.35$\pm$1.33\\
\textbf{DynaFed$^\dagger$}    &   \textbf{61.45$\pm$0.46} & \textbf{59.04$\pm$0.64}&   \textbf{68.35$\pm$0.20}   & \textbf{66.30$\pm$0.34} \\
\bottomrule
\end{tabular}
\vspace{-5pt}
\caption{Test performances on CIFAR10 achieved by different FL algorithms  under various degrees of data heterogeneity and local training epochs. Total communication rounds of 100 and 50 are set with local training epochs of 5 and 10, respectively. As can be seen, \ourmodel significantly surpasses other methods.}\label{tab:local_epoch}
\end{scriptsize}
\end{minipage}
\hspace{2mm}
\begin{minipage}[c]{0.48\textwidth}
\centering
\begin{scriptsize}
\begin{tabular}{l!{\vrule width 0.5pt}cc!{\vrule width 0.5pt}cc}
\toprule
 &        \multicolumn{2}{c}{CIFAR10}  & \multicolumn{2}{c}{CINIC10}\\
 & $\alpha = 0.01$&$\alpha = 0.04$&  $\alpha = 0.01$&$\alpha = 0.04$\\
Method& $Acc=0.45$&$Acc=0.55$& $Acc=0.33$& $Acc=0.45$\\
\midrule
FedAVG     &   132.0$\pm$15.0 & 117.0$\pm$8.0 &  189.3$\pm$10.5 & 138.7$\pm$5.6\\
FedProx    & 113.3$\pm$16.4 & 102.0$\pm$4.7 & 156.7$\pm$7.0 & 118.3$\pm$4.0\\
Scaffold    & 105.0$\pm$10.4 & 100.3$\pm$3.5 &  158.0$\pm$5.4 & 110.0$\pm$3.4\\
\midrule
FedDF$^*$   & 145.7$\pm$13.1 & 117.3$\pm$5.8 &  180.7$\pm$5.0 &170.0$\pm$7.0\\
FedBE$^*$   & 165.0$\pm$12.7 & 122.7$\pm$4.5 & 185.3$\pm$14.6   & 174.3$\pm$6.8 \\
ABAVG$^*$   & 109.7$\pm$5.4 & 110.7$\pm$5.0 & 150.0$\pm$8.4 & 127.0$\pm$5.8\\
\midrule
FedGen$^\dagger$    & 115.7$\pm$10.4 &110.3$\pm$5.7 &  167.0$\pm$12.1&128.3$\pm$5.5\\
FedMix$^\dagger$   & 77.3$\pm$3.7 & 89.3$\pm$3.5 & 79.0$\pm$7.8 & 82.3$\pm$5.5 \\
\textbf{DynaFed$^\dagger$}   & \textbf{22.3$\pm$0.6 }& \textbf{22.0$\pm$1.0} & \textbf{21.3 $\pm$1.5}  & \textbf{22.7$\pm$1.4}\\
\bottomrule
\end{tabular}
\vspace{-5pt}
\caption{Comparison of \textbf{the number of communication rounds} to reach target accuracy. With the knowledge of global data distribution stored in $\cD_\text{syn}$ at the server, the convergence speed of our \ourmodel is significantly accelerated.}\label{tab:Convergence_speed}
\vspace{-5mm}
\end{scriptsize}
\end{minipage}
\vspace{-5pt}
\end{table}

\vspace{-1em}
\paragraph{\ourmodel is Robust to Longer Local Training.} 
Longer local training is generally required in FL to reduce the total number of global communication rounds. Under different heterogeneity degrees, we conduct experiments to evaluate the impact of longer local training epochs on \ourmodel. Specifically, we conduct experiments with total communication rounds of 100 and 50  with ocal training epochs of 5 and 10, respectively. The results are presented in Table \ref{tab:local_epoch}, from which we observe the following: 1) \ourmodel consistently outperforms other methods by a large margin even with longer local training epochs; 2) the performance of \ourmodel is less sensitive to the length of local training, which benefits from the $\cD_\text{syn}$ containing information about the global data distribution. Therefore, \ourmodel is able to achieve similar performance with less global communication rounds, which is the major bottleneck in the efficiency of FL. 

We further conduct experiments on varying local epochs to measure the quality of $\cD_\text{syn}$. Specifically, we use it to train a network from scratch and evaluate its test performance. Shown in right of Figure \ref{fig:data_quality}, the quality of $\cD_\text{syn}$ stays similar with longer local training and more severe heterogeneity. This further explains the superior performance of \ourmodel with longer local training.

\begin{figure}[t!]
\small
\begin{minipage}[c]{0.48\textwidth}
    \centering
 \vspace{-0.25cm}
\begin{scriptsize}
\begin{tabular}{c!{\vrule width 0.5pt}cc!{\vrule width 0.5pt}cc}
\toprule
 &        \multicolumn{2}{c}{$\alpha=0.01$}  & \multicolumn{2}{c}{$\alpha=0.04$}\\
Method & MLP  &  ConvNet  &  MLP  &  ConvNet \\
\midrule
FedAVG     &   65.64$\pm$1.69&  74.51$\pm$1.32&   73.26$\pm$1.49& 81.74$\pm$1.98\\
FedProx   &  68.09$\pm$1.47& 76.88$\pm$1.83& 79.83$\pm$1.70 & 83.06$\pm$2.53\\
Scaffold   &  67.60$\pm$1.53& 77.92$\pm$0.87&   78.09$\pm$1.35&82.25$\pm$1.35\\
\midrule
FedDF$^*$    & 64.59$\pm$1.70& 72.36$\pm$2.08& 77.20$\pm$1.58  & 81.65$\pm$0.97\\
FedBE$^*$   & 65.97$\pm$1.64& 72.33$\pm$1.79 &   75.42$\pm$1.35 & 81.31$\pm$1.25\\
ABAVG$^*$    & 69.19$\pm$1.50& 75.98$\pm$1.99&  81.64$\pm$1.20 & 84.44$\pm$1.84\\
\midrule
FedGen$^\dagger$   & 68.67$\pm$1.45& 75.59$\pm$1.12&    77.94$\pm$1.38&56.60$\pm$1.08\\
FedMix $^\dagger$  & 70.30$\pm$0.92& 81.34$\pm$0.68&   81.95$\pm$0.64 & 84.23$\pm$0.50\\
\textbf{DynaFed$^\dagger$}    &   \textbf{73.89$\pm$0.24} &  \textbf{87.52$\pm$0.15} &    \textbf{83.54$\pm$0.42}   & \textbf{89.45$\pm$0.11}\\
\bottomrule
\end{tabular}
\vspace{-5pt}
\captionof{table}{Performance comparison across different network architectures. We conduct the experiment on FMNIST dataset using MLP and ConvNet to demonstrate the generalization of \ourmodel for different network architectures.}
\label{tab:cross_arch}
\end{scriptsize}
\end{minipage}
\hspace{2mm}
\begin{minipage}[c]{0.48\textwidth}
    \centering
    \begin{subfigure}[b]{0.48\textwidth}
        \centering
        \includegraphics[width=\textwidth]{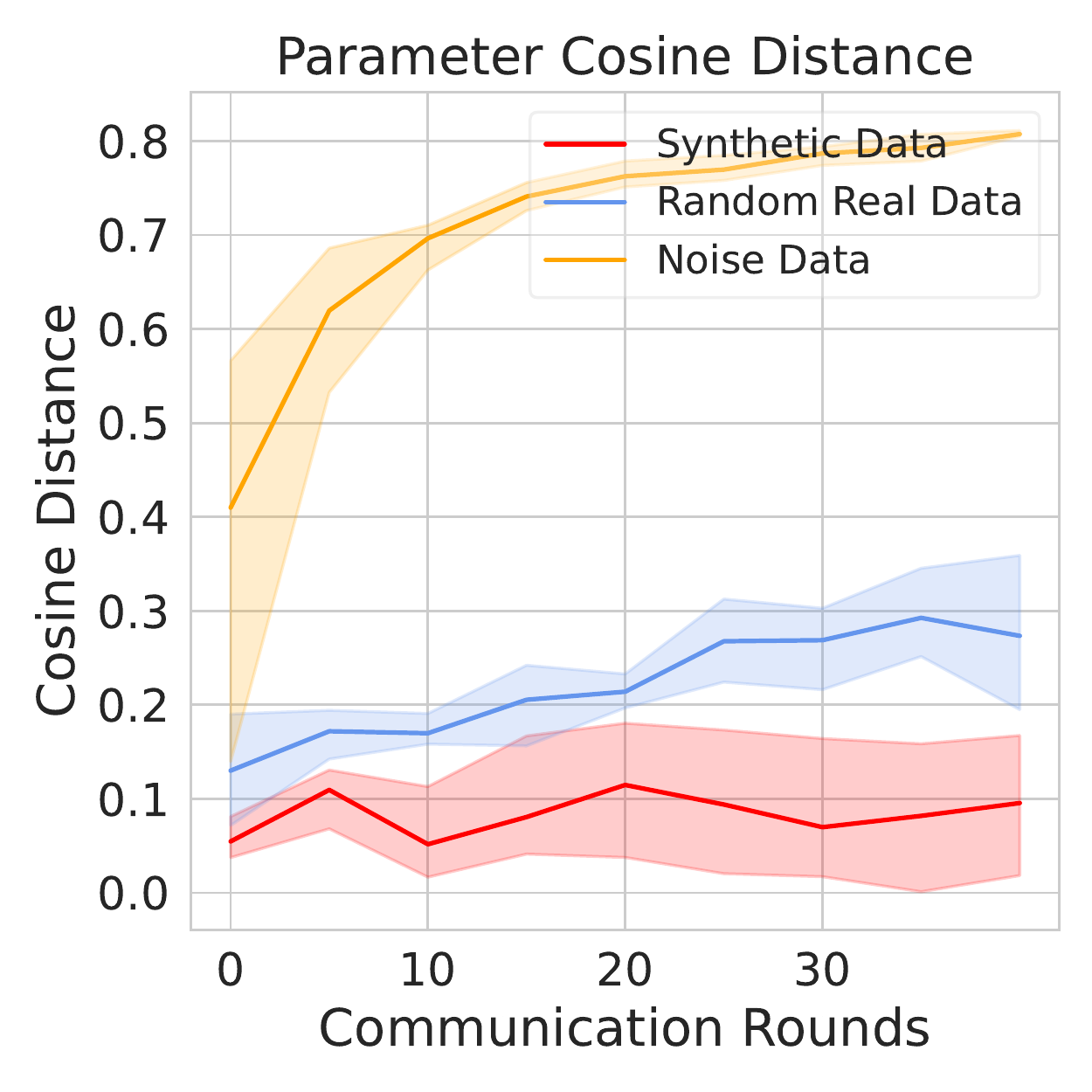}
    \end{subfigure}
    \hfill
    \begin{subfigure}[b]{0.48\textwidth}   
        \centering 
        \includegraphics[width=\textwidth]{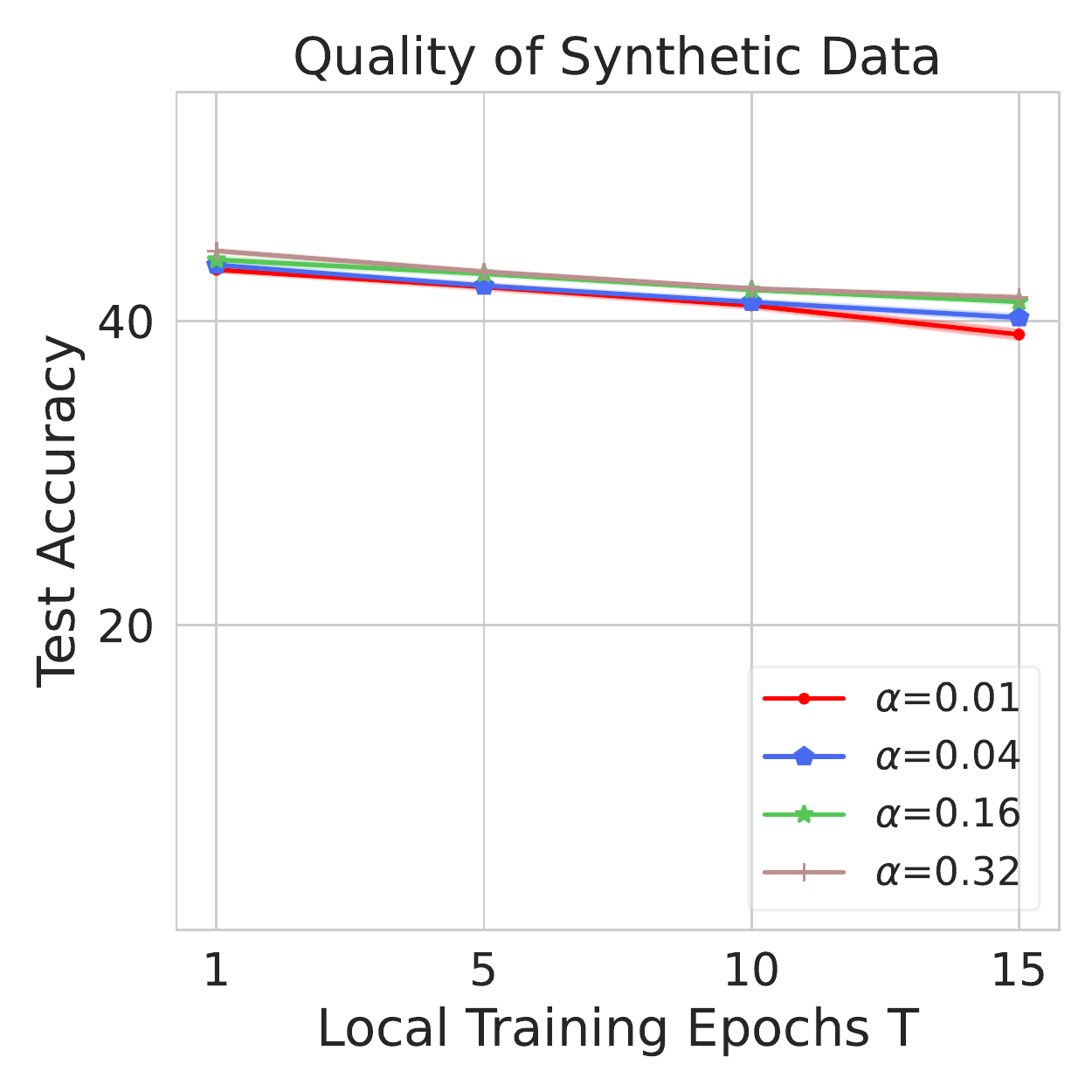}
    \end{subfigure}
    \hfill
    \vspace{-3mm}
    \caption[ The average and standard deviation of critical parameters ]
    {\small Left: the cosine distance between the target checkpoint $\bw^\text{t+s}$ and the parameters $\Tilde{\bw}$ obtained by training with different datasets from $\bw^\text{t}$ for $s'$ steps. We can see that the distance derived using $\cD_\text{syn}$ is constantly smaller compared with other data. Right: Test performances of a model trained from scratch using only $\cD_\text{syn}$ generated under various local epochs and client heterogeneity. The two plots together verify the quality of the generated $\cD_\text{syn}$.} 
    \label{fig:data_quality}
    \vspace{-3mm}
\end{minipage}
\end{figure}

\subsection{Architecture Generalization and Efficiency}
To showcase the generalization ability of our approach over different network architecture choices, we conduct experiments on FMNIST using different network architectures. Specifically, we choose ConvNet and MLP following \cite{zhu2021data, mcmahan2017communication}. As shown in Table \ref{tab:cross_arch}, under various client data heterogeneity, our method demonstrates superior performances with both network architectures.

Thanks to the rich knowledge of global data distribution contained in $\cD_\text{syn}$, using it to refine the global model greatly boosts the convergence speed of training. As shown in Figure \ref{fig:process_curve}, as soon as $\cD_\text{syn}$ is used to refine the global model, its performance rapidly increases to a reasonable accuracy, which reduces many rounds of communication. In Table \ref{tab:Convergence_speed}, we also quantitively compare the convergence speed of different FL algorithms by showing the number of communication rounds needed to reach the highest test accuracy achievable by the baselines. As can be observed, our \ourmodel requires only less than 20\% communication rounds to reach a target accuracy comparable to other methods.

\section{Conclusion}
In this paper, we propose a novel approach \ourmodel to tackle the data heterogeneity issue, which synthesizes a pseudo dataset to extract the essential knowledge of the global data distribution from the dynamics of the global model's trajectory. Extensive experiments show that \ourmodel demonstrates relative improvement over the the FedAvg baseline up to 82.2\% on CIFAR100. Further, we believe our work is able to provide insights for extracting global information on the server side, which goes beyond tackling the data heterogeneity issue.

{\small
\bibliographystyle{ieee_fullname}
\bibliography{egbib}
}
\onecolumn
\appendix
\section{Proofs of Theorem 1}

From lines 502 to 518, we rewrite the detailed steps of \ourmodel for the convenience of analysis. The key idea is to treat our aggregation as a local training process where each client works  with exactly the same parameters $\bw_m^t$ and gradients in each time step. And thus $\bar{\bw}^t =\bw_1^t=\cdots=\bw_M^t$ in the finetuning/aggregation process. Therefore, the convergence of $\bar{\bw}^t$ is actually that of our learned model. 

Firstly, we need the following lemmas.

\begin{assumption}\label{ass:variance}
The variance of the stochastic gradients in each client is bounded, i.e., 
$$\mathbb{E}\|\nabla \cL_m(\bw_m^t, \xi_m^t) - \nabla \cL_m(\bw_m^t, \cD_m)\|^2 \leq \sigma_m^2,$$
where $\xi_m^t$ is sampled from $\cD_m$ uniformly at random, $m=1,2,\ldots, M$. 
\end{assumption}

\begin{assumption}\label{ass:squared-norm}
The expectation of $\|\nabla \cL_m(\bw_m^t, \xi_m^t)\|^2$ is bounded, i.e., 
\begin{align}
    \mathbb{E}\|\nabla \cL_m(\bw_m^t, \xi_m^t)\|^2 \leq G^2,
\end{align}
where  $\xi_m^t$ is sampled from $\cD_m$ uniformly at random, $m=1,2,\ldots, M$. 
\end{assumption}

\begin{assumption}\label{ass:biased-sgd}
For our $\cD_\text{syn}$, we assume 
\begin{align}
   \| \nabla\cL(\bw, \cD_\text{syn}) - \nabla \cL(\bw, \cD)\| \leq \delta \|\nabla \cL(\bw, \cD)\| + \epsilon,
\end{align}
where $\delta \geq 0$ and $\epsilon \geq 0$ are two small scalars.
\end{assumption}

Assumptions \ref{ass:variance} and $\ref{ass:squared-norm}$ are widely used in stochastic optimization as well as FL\cite{li2019convergence}.  Assumption \ref{ass:biased-sgd} is a standard assumption for biased gradient \cite{hu2021analysis}. We restate Theorem 1 as follows:

\begin{theorem}[Convergence]\label{thm:convergence-appendix} Under Assumptions \ref{ass:variance}, \ref{ass:squared-norm} and \ref{ass:biased-sgd}, for $\tilde{L}$-smooth, $\mu$-strongly convex loss functions $\ell(\cdot, \cdot)$, we assume $\delta \tilde{L}< \mu$. Let $\eta_t = \frac{c}{t+\gamma}$ for a proper constant $c$ and $\gamma$. Then, \ourmodel satisfies 
\begin{align}
    \mathbb{E}\cL( \bar{\bw}^T, \cD) - \cL( {\bw}^*, \cD) \leq \frac{ C}{T},
\end{align}
where  $\bw^*$ is the minimum of $\cL(\bw, \cD)$ and  $C$ is a constant, whose detailed formula is given in Eqn.(\ref{eqn:C}).
\end{theorem}
Before giving the detailed proof, we need the following two lemmas. 

\begin{lemma}[One Step of Local Training]\label{lemma:1step-local-training} Under the same assumptions with Theorem \ref{thm:convergence-appendix}, for $t\in \cI$, it follows that
\begin{align}
    \mathbb{E}\|\bar{\bw}_{t+1} - \bw^* \|^2 \leq (1-\eta_t \mu)\mathbb{E}\|\bar{\bw}_{t} - \bw^* \|^2 + \eta_t^2 B_\text{loc}, \label{eqn:1step-localtraining}
\end{align}
where 
$$B_\text{loc}=\sum_{m=1}^M \alpha_m^2 \sigma_m^2 + 6\tilde{L}\Gamma +8(\tau_1-1)^2G^2 \mbox{ with } \Gamma = \cL(\bw^*, \cD) - \sum_{m=1}^M \alpha_m \min_{\bw} \cL_m(\bw, \cD_m).$$
\end{lemma}
\begin{proof} of Lemma \ref{lemma:1step-local-training}.
Noe that in \ourmodel, the local training is equivalent to that of FedAvg, therefore, the intermediate results in \cite{li2019convergence} hold for \ourmodel. 

The result in Eqn.(\ref{eqn:1step-localtraining}) can be directly obtained from the proof of Theorem 1 in \cite{li2019convergence}.
\end{proof}
 
In our aggregation, each step is essentially a biased gradient descent method, for which we have the following lemma from \cite{hu2021analysis}.

\begin{lemma}[One Step of Aggregation \cite{hu2021analysis}] \label{lemma:1step-aggregation} Under the same assumptions with Theorem \ref{thm:convergence-appendix}, for $t\notin \cI$, it follows that
\begin{align}
    \mathbb{E}\|\bar{\bw}_{t+1} - \bw^* \|^2 \leq \rho_t^2 \mathbb{E}\|\bar{\bw}_{t} - \bw^* \|^2 + \eta_t^2 B_\text{agg}, \label{eqn:1step-aggregation}
\end{align}
where $\rho^2_t = 1- (\mu-\delta \tilde{L}) \eta_t + \mathcal{O}(\eta_t^2)$ and  $B_\text{agg}$ is a positive constant.
\end{lemma}
The lemma above is actually one of the implications of Theorem 1 in  \cite{hu2021analysis}. Please refer to Sections 1 and 3.2 of \cite{hu2021analysis} for more details.  
\begin{remark}
Note that when $t\notin \cI$, we perform model aggregation. Therefore, Lemma \ref{lemma:1step-aggregation} demonstrates that each step in our aggregation process can reduce the expectation of the squared distance between $\bar{\bw}_{t}$ and the optimal solution $\bw^*$, i.e., our aggregation can improve equality of the aggregated model during fintuning. This is benefited from our synthesized data $\cD_\text{syn}$. It is consistent with our exmperimental results in Figure 3 in the main paper.
\end{remark}

Now we turn to prove Theorem \ref{thm:convergence-appendix}. 
\begin{proof} of Theorem \ref{thm:convergence-appendix}:
From Eqn.(\ref{eqn:1step-aggregation}), we know when $\eta_t  = \frac{c}{t+\gamma}$ is sufficiently small, which can be satisfied by choosing proper $c$ or sufficiently large $\gamma$, the following holds for $t\notin \cI$: 
\begin{align}
      \mathbb{E}\|\bar{\bw}_{t+1} - \bw^* \|^2 \leq (1-\tilde{\mu} \eta_t) \mathbb{E}\|\bar{\bw}_{t} - \bw^* \|^2 + \eta_t^2 B_\text{agg},\label{eqn:1step-aggregation-1}
\end{align}
where $\tilde{\mu}$ can be any positive constant with $\tilde{\mu}< \mu - \delta \tilde{L}$.

By defining $\tilde{B} = \max\{B_\text{loc}, B_\text{agg}\}$, we can unify the Eqn.(\ref{eqn:1step-localtraining}) and Eqn.(\ref{eqn:1step-aggregation}) into  
\begin{align}
      \mathbb{E}\|\bar{\bw}_{t+1} - \bw^* \|^2 \leq (1-\tilde{\mu} \eta_t) \mathbb{E}\|\bar{\bw}_{t} - \bw^* \|^2 + \eta_t^2 \tilde{B}, \mbox{ for all } t.
\end{align}
We assume $c\tilde{\mu}>1$, which can be satisfied by choosing proper $\gamma$ to make Eqn.(\ref{eqn:1step-aggregation-1}) hold.  Let 
\begin{align}
    v=\max\{\frac{c^2\tilde{B}}{c\tilde{\mu}-1}, (\gamma+1)\mathbb{E}\|\bar{\bw}_{1} - \bw^* \|^2\}. \label{eqn:v}
\end{align}
We claim that 
\begin{align}
    \mathbb{E}\|\bar{\bw}_{t} - \bw^* \|^2< \frac{v}{t+\gamma}. \label{eqn:colcusion}
\end{align} We prove this claim by induction. 
Firstly Eqn.(\ref{eqn:colcusion}) hold for $t=1$ due to the definition of $v$. Assume it also holds for some $t\geq 1$, we have 
\begin{align*}
    \mathbb{E}\|\bar{\bw}_{t+1} - \bw^* \|^2 &\leq (1-\tilde{\mu} \eta_t) \mathbb{E}\|\bar{\bw}_{t} - \bw^* \|^2 + \eta_t^2 \tilde{B}\nonumber\\
    &\leq (1-\tilde{\mu} \eta_t)\frac{v}{t+\gamma} + \eta_t^2 \tilde{B} \nonumber\\
    & \leq (1- \frac{\tilde{\mu}c}{t+\gamma}) \frac{v}{t+\gamma} + \frac{c^2}{(t+\gamma)^2} \tilde{B} \nonumber\\
    & \leq \frac{t+\gamma-1}{(t+\gamma)^2} v + \frac{1}{(t+\gamma)^2} \left(c^2\tilde{B} - (\tilde{\mu}c-1)v\right) \nonumber\\
    & \leq \frac{t+\gamma-1}{(t+\gamma)^2} v\nonumber\\
    & \leq \frac{v}{t+1+\gamma}.\nonumber
\end{align*}
Therefore, our claim holds for all $t>0$. 

Hence, for $\tilde{L}$-smoothness of $\cL$, we have 
\begin{align}
   \mathbb{E}\left[\cL(\bar{\bw}^T, \cD) -  \cL(\bar{\bw}^*, \cD)\right] \leq \frac{\tilde{L}}{2}\mathbb{E}\|\bar{\bw}_{T} - \bw^* \|^2 \leq \frac{ \tilde{L}v}{2(T+\gamma)} \leq \frac{C}{T},\nonumber
\end{align}
where 
\begin{align}
    C = \frac{\tilde{L}v}{2}.\label{eqn:C}
\end{align}
\end{proof}

\begin{remark}
We give the result of full device participation setting above. We would like to point out that for the setting of partial device participation, we can have a similar result with different $C$. The reason is the only difference with the full device participation setting is the local training process and the aggregation process is kept the same. The Eqn.(21) in \cite{li2019convergence} shows that Lemma \ref{lemma:1step-local-training} holds with a different constant $B$ in this setting. To avoid redundancy, we omit this result in this appendix and please refer to \cite{li2019convergence} for more details.
\end{remark}

\section{Detailed Experiment Settings}
\paragraph{Detailed Configurations} Unless specified otherwise, we follow \cite{Fedaux, chen2021bridging, xie2022optimizing} and adopt the following default configurations throughout the experiments: we run 200 global communication rounds with local epoch set to 1. In the experiments in the main paper,  a total of 80 clients are adopted, and the participation ratio in each round is set to 40\%. Experiments using 20\% participation ratio are in Table \ref{tab:cifar10_exp_pr02} and Figure \ref{fig:cifar10_process_pr02}. For the network choices, we use ConvNet~\cite{lecun1998gradient} with 3 layers, and the hidden dimension is set to 128. 
The local learning rate is set to $10^{-3}$ with Adam optimizer\cite{kingma2014adam}. 
We report the global model's average performance in the last five rounds  evaluated using the test split of the datasets. For the construction of global trajectory, we first run FedAvg~\cite{mcmahan2017communication} and use the checkpoints from the first 20 communication rounds ($L=20$). 
For the data synthesis process on the server, we use Adam optimizer with learning rate $5\times10^{-2}$ for optimizing the data and the label in the outer loop, SGD with learning rate $10^{-5}$ is used to train the network on the synthetic data.
We set the time difference $s$ between the start and end checkpoint to 5. The total number of iteration for synthesis is $N=1000$, which takes around 1 GPU hour on RTX 3080 Ti. The subsequent finetuning takes negligible time since the amount of synthesized data is set to 150.

\subsection{Detailed Descriptions of Baselines}
(1) FedAVG~\cite{mcmahan2017communication}: The most prevalent aggregation strategy for federated learning, which simply averages the weights sent by the clients to form the global model. (2)FedPROX~\cite{li2020federated}: a method that alleviates the client heterogeneity by regularizing the drift of local models with the global model via a penalty term during local training. (3)Scaffold~\cite{karimireddy2019scaffold}: a method for client heterogeneous that introduces control variates to current local gradients and performs variance reduction to stabilize training. (4)FedDF~\cite{lin2020ensembleFedDF}: a method using knowledge distillation with unlabelled server data, which distills the ensemble knowledge from the client ensemble into the global model.(5)FedBE~\cite{chen2020fedbe}: a method based on FedDF, which uses bayesian ensemble-based knowledge distillation with unlabelled server data.(6) ABAVG~\cite{xiao2021novel}: An method using validation accuracy to reweight the clients, which needs labelled server data. (7) FedGen \cite{zhu2021data}: A data-free knowledge distillation method that trains a generator that generates an embedding based on the class label. This generator is trained using the global model as the discriminator. The generator is sent to the clients to help local training.
(8) FedMix\cite{yoon2021fedmix}: a data-sharing strategy that uses linear combination between data points to preserve client privacy, which still has the potential to leak privacy, as stated in the original paper.

\section{Experiment with Other Participation Ratio}
To verify our \ourmodel is able to work well under different client participation ratios, we conduct additional experiments on CIFAR10 with the participation ratio set to 20\%. As shown in Table \ref{tab:cifar10_exp_pr02} and Figure \ref{fig:cifar10_process_pr02}, our method demonstrates superior performances against other baseline methods.

 \begin{table}[t!]
    \centering
 \vspace{-0.25cm}
\begin{tabular}{cccc}
\toprule
Method &   $\alpha=0.01$ & $\alpha=0.04$& $\alpha=0.16$ \\
\midrule
FedAVG     &   33.62$\pm$4.36 & 50.36$\pm$3.10&68.05$\pm$1.39\\
FedProx   &  39.87$\pm$2.34  & 52.78$\pm$2.69 & 69.99$\pm$1.08\\
Scaffold    &   38.65$\pm$2.21  & 53.13$\pm$1.35 & 70.01$\pm$0.87\\
\midrule
FedDF$^*$     & 35.25$\pm$2.90  & 50.02$\pm$3.34&  68.82$\pm$1.07\\
FedBE$^*$   & 29.98$\pm$3.02  & 48.97$\pm$3.86&   68.84$\pm$0.96\\
ABAVG    &  37.26$\pm$2.89 & 57.88$\pm$0.78&   72.05$\pm$0.88\\
\midrule
FedGen$^\dagger$   &  36.28$\pm$3.54 & 52.11$\pm$2.36&70.17 $\pm$1.20 \\
FedMix$^\dagger$   & 46.77$\pm$1.93 & 59.80$\pm$1.34&  70.59$\pm$0.31\\
\textbf{DynaFed$^\dagger$}    &   \textbf{62.53$\pm$0.57} & \textbf{67.54$\pm$0.44}   & \textbf{73.59$\pm$0.12}\\
\bottomrule
\vspace{-5mm}
\end{tabular}
\caption{Comparison of test performances on CIFAR10 with different degrees of data heterogeneity $\alpha$. The client participation ratio per round is set to 20\%. Our \ourmodel outperforms other approaches by a large margin, and the superiority is more evident under more severe heterogeneity. Specifically, \ourmodel has a relative performance gain of 86\% over the FedAvg baseline when $\alpha=0.01$. }\label{tab:cifar10_exp_pr02}
\end{table}

\begin{figure}
    \centering
    \begin{subfigure}[b]{0.3\textwidth}
        \centering
        \includegraphics[width=\textwidth]{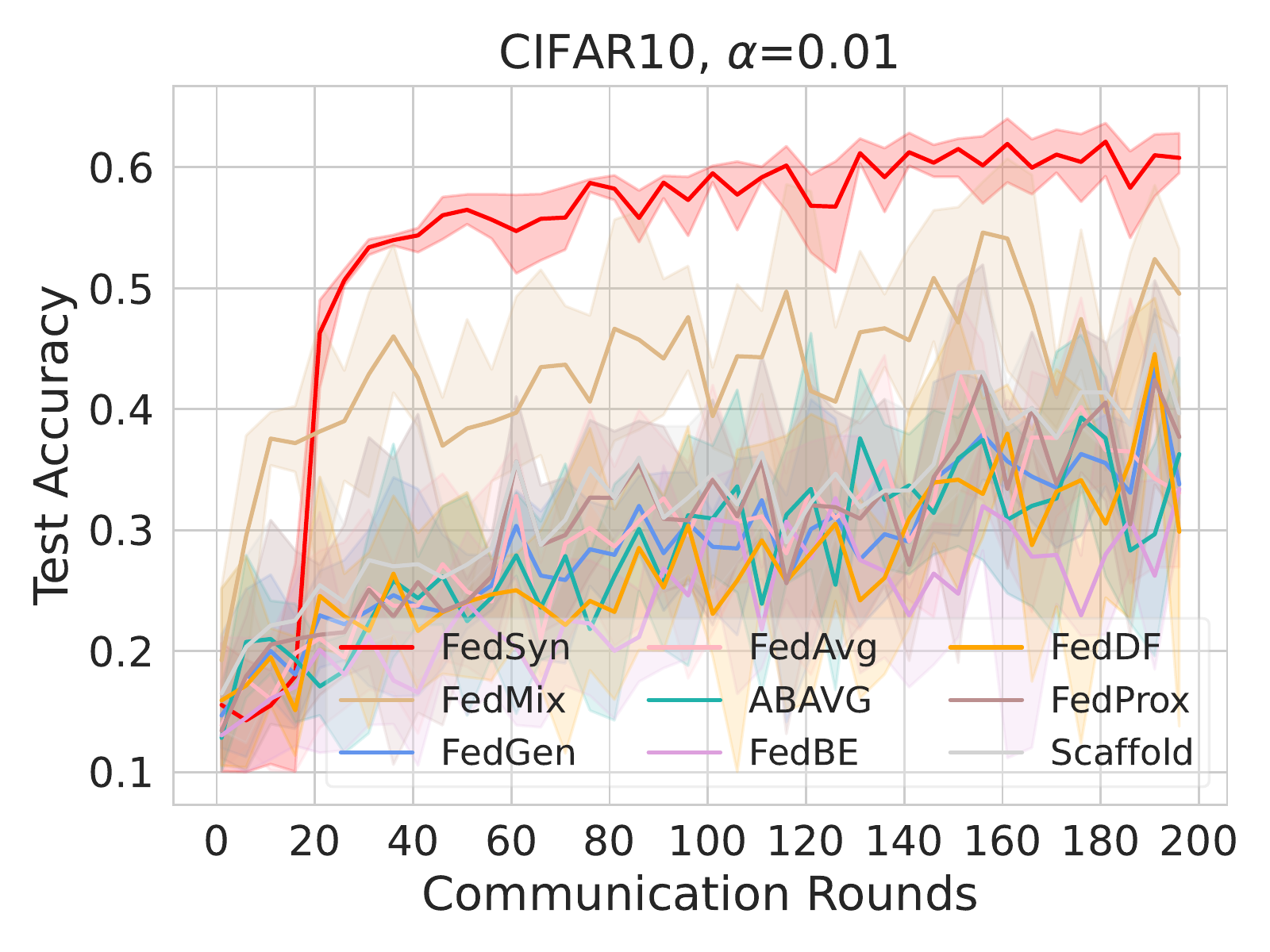}
    \end{subfigure}
    \begin{subfigure}[b]{0.3\textwidth}
        \centering
        \includegraphics[width=\textwidth]{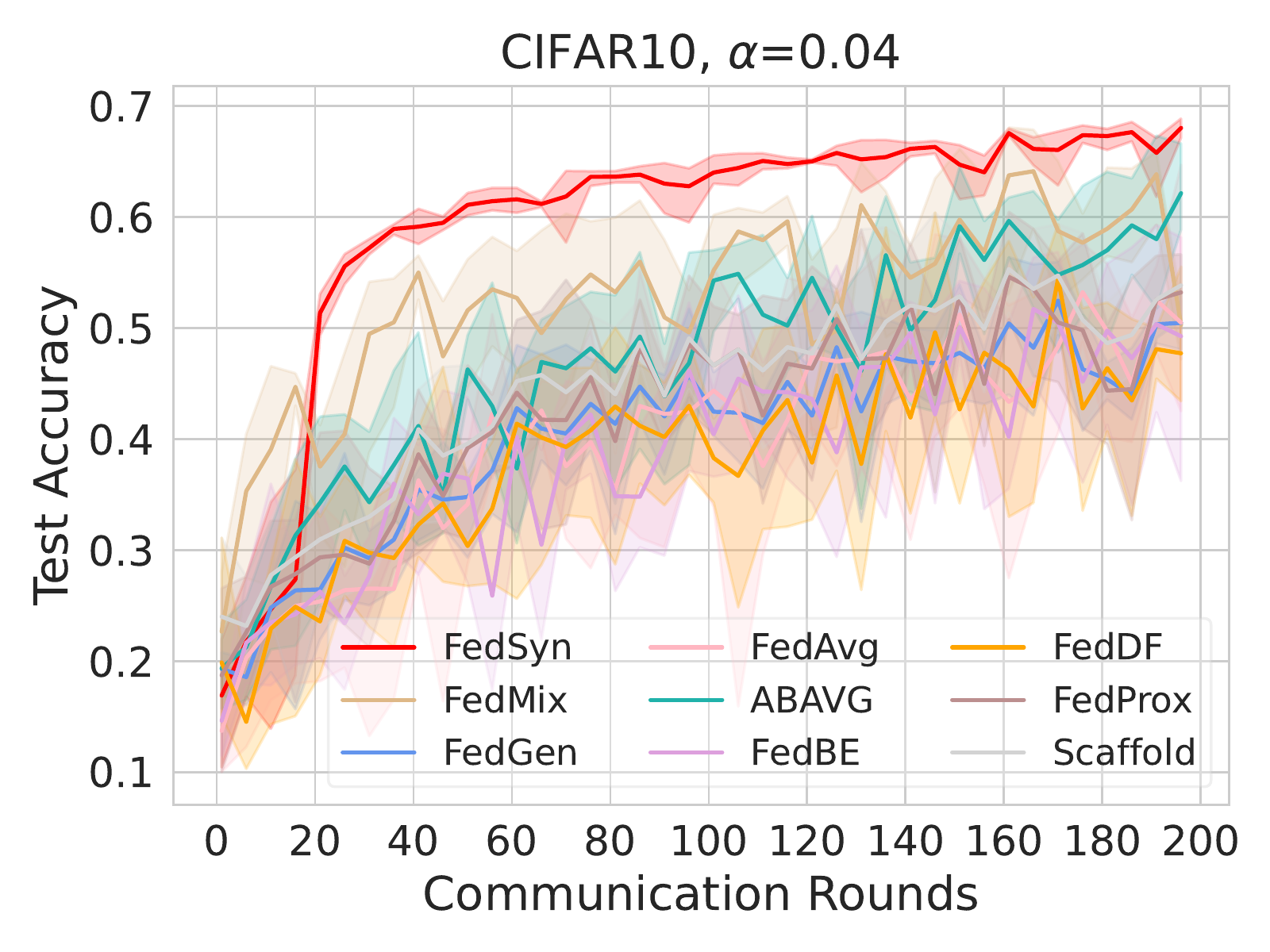}
    \end{subfigure}
    \begin{subfigure}[b]{0.3\textwidth}   
        \centering 
        \includegraphics[width=\textwidth]{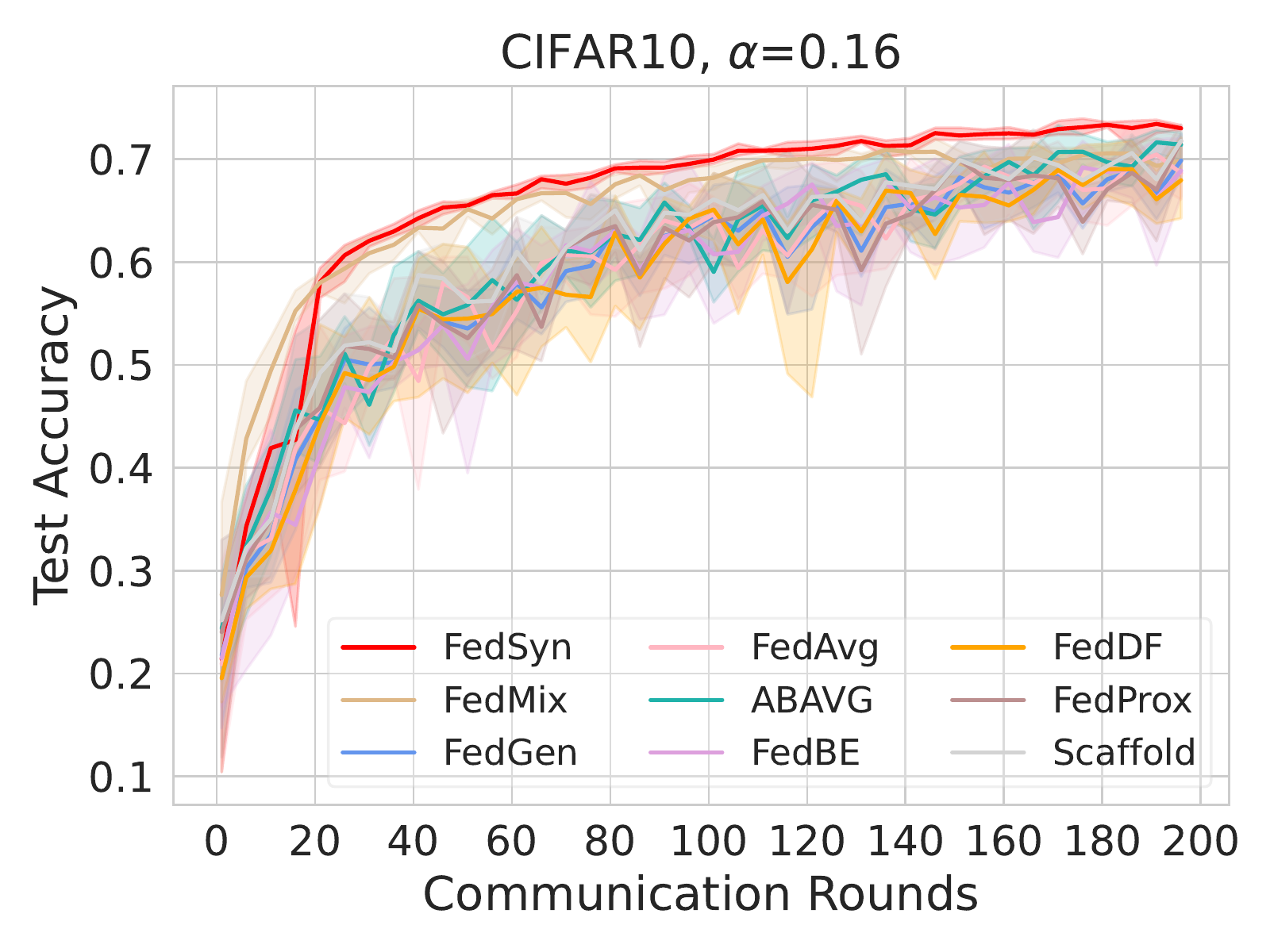}
    \end{subfigure}
    \caption[ The average and standard deviation of critical parameters ]
    {\small Visualization of global model's test performance on CIFAR10 with 20\% participation ratio throughout the global communication rounds. We can see that the global model rapidly converges to a satisfactory test accuracy once $\cD_\text{syn}$ participates in refining the global model. Furthermore, $\cD_\text{syn}$ also helps reduce the fluctuation of model performances between communication rounds, which significantly boosts the training stability. \ourmodel requires less than 10\% communication rounds to achieve comparable performance with the baseline methods.} 
    \label{fig:cifar10_process_pr02}
    \vspace{-5mm}
\end{figure}
\end{document}